\documentclass[noinfoline]{imsart}

\usepackage{amssymb}

\usepackage{mathrsfs}
\usepackage{latexsym}
\usepackage{amsmath}
\usepackage{amsthm}
\usepackage{amsfonts}
\usepackage{url}
\usepackage{amssymb}
\usepackage{enumerate}%
\usepackage{capt-of}%

\usepackage
[final]
{showkeys}
\usepackage{calrsfs}

\usepackage{mathtools}

\usepackage{todonotes}

\makeatletter
\DeclareRobustCommand\widecheck[1]{{\mathpalette\@widecheck{#1}}}
\def\@widecheck#1#2{%
    \setbox\z@\hbox{\m@th$#1#2$}%
    \setbox\tw@\hbox{\m@th$#1%
       \widehat{%
          \vrule\@width\z@\@height\ht\z@
          \vrule\@height\z@\@width\wd\z@}$}%
    \dp\tw@-\ht\z@
    \@tempdima\ht\z@ \advance\@tempdima2\ht\tw@ \divide\@tempdima\thr@@
    \setbox\tw@\hbox{%
       \raise\@tempdima\hbox{\scalebox{1}[-1]{\lower\@tempdima\box
\tw@}}}%
    {\ooalign{\box\tw@ \cr \box\z@}}}
\makeatother

\newlength{\mylength}
\makeatletter
\newcommand{\mycfs}[1]{%
  \normalsize
  \@defaultunits\mylength=#1pt\relax\@nnil
  \edef\@tempa{{\strip@pt\mylength}}%
  \ifx\protect\@typeset@protect
     \edef\@currsize{\noexpand\mycfs\@tempa}
  \fi
  \mylength=1.2\mylength
  \edef\@tempa{\@tempa{\strip@pt\mylength}}%
  \expandafter\fontsize\@tempa
  \selectfont
}
\makeatother

\startlocaldefs

\theoremstyle{plain} 
\newtheorem{theorem}{Theorem}[section]

\newtheorem*{thA}{Theorem~A}
\newtheorem*{thB}{Theorem~B}
\newtheorem{lemma}[theorem]{Lemma}

\newtheorem{proposition}[theorem]{Proposition}

\theoremstyle{definition} 

\theoremstyle{definition} 

\theoremstyle{remark} 

\theoremstyle{remark} 
\newtheorem{remark}[theorem]{Remark}

\numberwithin{equation}{section}

\newcommand{\texorpdfstring}[2]{#1}

\newcommand{\hide}[1]{}
\newcommand{\pred}[1]{\ii{#1}}

\newcommand{\E}{\operatorname{\mathsf{E}}}
\renewcommand{\P}{\operatorname{\mathsf{P}}}
\newcommand{\A}
{\mathop{\operatorname{ave}}}
\newcommand{\Ax}
{\mathop{\operatorname{ave}'_x}}

\newcommand{\cc}{\mathsf{c}}
\newcommand{\emp}{\mathsf{
ERM}}
\newcommand{\erm}{\mathsf{ERM}}
\newcommand{\LB}{\mathsf{LB}}
\newcommand{\UB}{\mathsf{UB}}
\newcommand{\plo}{P_\mathsf{low}}

\newcommand{\erf}{\operatorname{erf}}
\newcommand{\sign}{\operatorname{sgn}}
\newcommand{\card}{\operatorname{card}}

\newcommand{\tsuml}{\mathop{\textstyle{\sum}}\limits}

\newcommand{\Rad}
{R}
\newcommand{\opt}{\operatorname{bayes}}
\newcommand{\vc}[1]{{\operatorname{
VC(#1)}}}
\newcommand{\hatopt}{\widecheck{\opt}}

\newcommand{\vbias}{\boldsymbol{\bias}}

\newcommand{\hopt}{h_{\vbias}}

\def\bias
{b}

\renewcommand{\vbias}{\gamma}
\renewcommand{\vbias}{\g}

\newcommand{\g}
{\beta}

\newcommand{\al}{\alpha}

\newcommand{\De}{\Delta}
\newcommand{\ka}{\kappa}
\newcommand{\la}{\lambda}
\newcommand{\si}{\sigma}

\newcommand{\intr}[2]{\overline{{#1},{#2}}}
\newcommand{\ii}[1]{\operatorname{I}\{#1\}}

\newcommand{\basicspace}{\calX}
\newcommand{\X}{\basicspace}
\newcommand{\Y}{\calY}
\newcommand{\Z}{\calZ}
\renewcommand{\H}{\calH}

\newcommand{\RR}{\mathfrak{R}}

\newcommand{\D}{\mathcal{D}}

\newcommand{\calH}{\mathcal{H}}

\newcommand{\J}{\mathcal{J}}

\newcommand{\calX}{\mathcal{X}}
\newcommand{\calY}{\mathcal{Y}}
\newcommand{\calZ}{\mathcal{Z}}

\newcommand{\ZZ}{\mathbb{Z}}

\newcommand{\beq}{\begin{eqnarray*}}
\newcommand{\eeq}{\end{eqnarray*}}
\newcommand{\beqn}{\begin{eqnarray}}
\newcommand{\eeqn}{\end{eqnarray}}
\newcommand{\bepf}{\begin{pf}}
\newcommand{\enpf}{$\Box$ \end{pf}}

\newcommand{\paren}[1]{\left( #1 \right)}

\newcommand{\ben}{\begin{enumerate}}
\newcommand{\een}{\end{enumerate}}
\newcommand{\bit}{\begin{itemize}}
\newcommand{\eit}{\end{itemize}}

\def\clap#1{\hbox to 0pt{\hss#1\hss}}

\newcommand{\err}{\operatorname{err}}

\newcommand{\vp}{\varepsilon}

\renewcommand{\L}{\mathcal{L}}
\newcommand{\Lrand}{\L_{\mathsf{rand}}}
\newcommand{\Lrandm}[1]{\L_{\mathsf{rand},#1}}

\newcommand{\tB}{\tilde{B}}
\newcommand{\tc}{\tilde{c}}
\newcommand{\tC}{\tilde{C}}
\newcommand{\tD}{{\tilde{D}}}
\newcommand{\ts}{\tilde{s}}
\newcommand{\tS}{\tilde{S}}
\newcommand{\tX}{\tilde{\X}}

\newcommand{\citep}{\cite}
\newcommand{\citet}{\cite}

\endlocaldefs

\begin{document}

\pagenumbering{Alph}

\begin{frontmatter}

\title{Exact Lower Bounds for the Agnostic Probably-Approximately-Correct (PAC) Machine Learning Model
}
\runtitle{Exact Agnostic PAC Lower Bounds}

\begin{aug}
\author{\fnms{Aryeh} \snm{Kontorovich}
\ead[label=e1]{karyeh@cs.bgu.ac.il}
} 
\and
\author{\fnms{ Iosif} \snm{Pinelis}
\ead[label=e2]{ipinelis@mtu.edu}
}

\thankstext{t1}{\today}
\runauthor{A. Kontorovich and I. Pinelis}

\affiliation{Ben-Gurion University\thanksmark{m1} and Michigan Technological University\thanksmark{m2}}

\address{Department of Computer Science\\
Ben-Gurion University\\
Beer Sheva, Israel 84105\\
\printead{e1}\\
}

\address{Department of Mathematical Sciences\\
Michigan Technological University\\
Houghton, Michigan 49931-1295 
U.S.A.\\
\printead{e2}\\
}
\end{aug}

\begin{abstract} 
We provide an exact non-asymptotic lower bound on the minimax expected excess risk (EER) in the agnostic probably-ap\-proximately-correct (PAC) machine learning classification model and identify 
minimax learning algorithms as certain maximally symmetric and minimally randomized ``voting'' procedures. 
Based on this result, an exact asymptotic lower bound on the minimax EER is provided.  
This bound is of the simple form $c_\infty/\sqrt{\nu}$ as $\nu\to\infty$, where $c_\infty=0.16997\dots$ is a universal constant, $\nu=m/d$, $m$ is the size of the  
training sample, and $d$ is the Vapnik--Chervonenkis dimension of the hypothesis class. 
It is shown that the differences between these asymptotic and non-asymptotic bounds, as well as the differences between these two bounds and the maximum EER of any learning algorithms that minimize the empirical risk, are asymptotically negligible, and all these differences are due to ties in the mentioned ``voting'' procedures. A few easy to compute non-asymptotic lower bounds on the minimax EER are also obtained, which are shown to be close to the exact asymptotic lower bound $c_\infty/\sqrt{\nu}$ even for rather small values of the ratio $\nu=m/d$. 
As an application of these results, we substantially improve existing lower bounds on the tail probability of the excess risk. 
Among the tools used are  
Bayes estimation and apparently new identities and inequalities for binomial distributions.  
\end{abstract}

\setattribute{keyword}{AMS}{AMS 2010 subject classifications:}


		
		

 	


\end{frontmatter}

\pagenumbering{arabic}



\section{Introduction}
\label{sec:intro}

The Probably Approximately Correct (PAC) model aims at providing a clean, plausible
and minimalistic
abstraction of the supervised learning process~\citep{
MR0288823,valiant84}.
In this paper
we are concerned with the version of this model most commonly appearing in modern literature,
{\em agnostic PAC}
\citep{
haussler92,kearns-etal,183126}.

\label{p_intro}
Let $\X$ be an arbitrary nonempty set. 
The objective is to classify the elements of the set $\X$ 
into two classes, by attaching a label $1$ or $-1$ to each $x\in\X$. 
Let $\Y:=\{-1,1\}$, the set of labels. 
Then a possible classification 
rule may be identified with a map $h\colon\X\to\Y$, called a
{\em hypothesis}.
Usually,
hypotheses
are restricted to be elements of 
a specified subset $\H$ of the set $\Y^\X$ of all maps of $\X$ to $\Y$; this subset $\H$ is called the \emph{hypothesis class}. 

It is assumed that there exists a true (but unknown to us) probability distribution, say $D$, on the set $\X\times\Y$ of all pairs $(x,y)$ with $x\in\X$ and $y\in\Y$. 
To avoid tedious matters of measurability, let us just assume that the set $\X$ is finite. 

In the agnostic PAC model, considered in this paper, it is assumed that the distribution $D$ may be of completely arbitrary form, and the only information about it is provided to us by the ``sample'' values of a {\em labeled sample} $(X^{D}_1,Y_1^{D}),\dots,(X_m^{D},Y_m^{D})$ of $m$ independent copies of a random pair \break 
$(X^{D},Y^{D})$; here and in what follows, the superscript indicates the distribution of the random pair.  

%
The classification error probability for
a
hypothesis
$h\in\H$ is 
\begin{equation}\label{eq:err-def}
	\err(h,D):= \P(h(X^D)\neq Y^D).  
\end{equation}  
It should 
be clear that 
the least possible error probability
\begin{equation}\label{eq:min-err-def}
	\err_{\min}(D)
	:=\err_{\min,\H}(D):=\min_{h\in\H}\err(h,D)
\end{equation}
will 
usually be strictly greater than $0$, even when the true distribution $D$ is known. 


In the agnostic PAC model, considered here, the only information about the unknown distribution $D$ is provided by the  values of the sequence of $m$ independent random pairs 
\begin{equation}\label{eq:Z_m}
Z_m^D:=\big((X^{D}_1,Y_1^{D}),\dots,(X_m^{D},Y_m^{D})\big). 
\end{equation}
\label{L}
Therefore, the available ``learning'' strategies are the mappings 
\begin{equation*}
L\colon(\X\times\Y)^m\to\H, 	
\end{equation*}
called \emph{learning algorithms}. 

Let $h_D$ denote any minimizer of $\err(h,D)$
over
$h\in\H$. Of course, $h_D$ is unknown, since the distribution $D$ is unknown. However, it may be reasonable to use the
plug-in 
estimator $h_{\hat D_m}$ of $h_D$, obtained by substituting for $D$ the empirical distribution $\hat D_m=\hat D_m(z_m)$ based on a ``realization''
\begin{equation}\label{eq:z_m}
z_m:=\big((x_1,y_1),\dots,(x_m,y_m)\big)\in(\X\times\Y)^m 	
\end{equation}
of the ``random sample'' $Z_m$ from the distribution $D$. That is, 
\begin{equation*}
h_{\hat D_m}=
L_\emp\big(Z_m^D\big), 	
\end{equation*}
\noindent 
where $L_\emp$ is 
\label{L_emp} %
an empirical risk minimizer, that is, 
any learning algorithm such that for each
given sequence 
$z_m\in(\X\times\Y)^m$, the corresponding value
$L_\emp(z_m)$ of $L_\emp$ is a minimizer in $h\in\H$ of the ``empirical risk'' 
\begin{equation*}
	\err(h,\hat D_m)=\frac1m\,\sum_{i=1}^m\pred{h(x_i)\ne y_i}.
\end{equation*}
Such a minimizer
need not 
be unique, and so, the ``empirical minimization'' learning algorithm $L_\emp$ does not have to be unique. 

A
nontrivial 
question to ask here is how well
the empirical risk minimizer $h_{\hat D_m}$
performs
compared to the best possible
hypothesis,
$h_{D}$ --- that is,
how large the \emph{excess risk} $\De(h_{\hat D_m},D)$ is, where  
\begin{equation}
  \label{eq:Delta}
	\De(h,D) := \err(h,D)-\err_{\min}(D)=\err(h,D)-\err(h_D,D).
\end{equation}

This question has been
to a large extent resolved.
In particular, Theorem~4.9 from \citep{MR1741038} (slightly restated here) provides the following upper bound on the tail probabilities for the excess risk. 
\begin{thA}
  \label{thm:upper}
  There is a universal 
  real constant $c>0$ such that
  for all finite sets $\X$, 
  all distributions $D$ on $\X\times\Y$,
  all sample sizes $m
  $,
  and all hypothesis classes 
  $\H\subseteq\Y^\X$ of VC dimension $d
  $, 
  we have 
\begin{equation}\label{eq:upper}
  \P\paren{ \De\big(
  L_\emp(Z_m^D),D\big) \ge cu} \le 
  \exp\{-(mu^2-d)_+\}	
\end{equation} 
for all real $u\ge0$,   
where $z_+:=0\vee z$ for real $z$.   
\end{thA}
See \citet{MR1741038}
for an account of the intermediate steps leading up to the 
highly non-trivial result presented in Theorem~A; 
milestones here
include 
the seminal paper \citet{MR0288823} by Vapnik and Chervonenkis, 
\label{VC71}
followed, notably, by 
work of 
Talagrand~\citet{talagrand1994}, 
Haussler~\citet{MR1313896}, and 
Long~\citet{Long1999}.


Recall that the VC dimension (that is, the Vapnik--Chervonenkis dimension) of a set $\H\subseteq\Y^\X$ is the largest nonnegative integer $k$ such that there is a subset of $\X$ of cardinality $k$ that is shattered by $\H$; and a subset $\X_0$ of $\X$ is said to be shattered by $\H$ if the set of the restrictions to $\X_0$ of all the functions $h\in\H$ coincides with the entire set $\Y^{\X_0}$ of all functions from $\X_0$ to $\Y$. 

In what follows, $d$ will always denote $\vc\H$, the VC dimension of $\H$. 
\label{VC}
The case $d=0$ may occur only if the cardinality of $\H$ is at most $1$, so that there is at most one hypothesis to choose. This trivial case will be excluded in the sequel; that is, we shall assume that 
$d=1$. 
Then, in particular, one can introduce the
fundamental 
ratio 
\begin{equation}\label{eq:nu:=}
	\nu:=
	m/d  
\end{equation}
of the sample size $m$ to the VC dimension $d$.

Lower bounds 
matching, up to constant factors, the upper bound given in Theorem~A 
are also known. The one with the apparently best
currently 
known
numerical constants was given in
\citep[Theorem 5.2]{MR1741038}, which can be restated as follows. 

\begin{thB}
  \label{thm:lower}
  If $\nu=m/d\ge64^2/320=12.8$, then
  for 
  any finite set
  $\X$, 
any
hypothesis class
  $\H\subseteq\Y^\X$ of VC dimension $d$,
  and
  any learning algorithm $L\colon(\X\times\Y)^m\to\H$,
  there is a distribution $D$ on $\X\times\Y$ such that 
\begin{equation}\label{eq:lower}
	\P\Big(\De\big(%
	L(Z_m^D),D\big) > \frac1{\sqrt{320\nu}}\Big)\ge\frac1{64}. 
\end{equation}
\end{thB}

This lower bound is also the culmination of a notable historical development \cite{MR0288823,devr-lug95,simon96},
detailed in~\citet{MR1741038}.

\begin{remark}
  In~\citet[Section 28.2.2]{shwartz2014understanding} a much better constant factor, $1/8$, was claimed in place of
  $320$
  in (\ref{eq:lower}).
However, there is a mistake
  in the calculation; the actual value of the constant furnished by the proof is $512$.
\end{remark}

Introduce the \emph{expected excess risk} 
(EER) 
\begin{equation}\label{eq:R}
	\RR(L,D):=\RR_m(L,D):=\E\De\big(L(Z_m^D),D\big). 
\end{equation}

%

Let $\D:=\D_{\X
}$ and $\L:=\L_{\X
;m,\H}$ denote, 
\label{LL}
respectively, the set of all distributions on $\X\times\Y$ and the set of all learning algorithms $L\colon(\X\times\Y)^m\to\H$; recall here that $\Y=\{-1,1\}$. 
Let then 
\begin{equation}\label{eq:ub,lb}
\begin{aligned}
	c_{m,d}^\UB:=\sqrt{
	m/d}\,\sup_\X\sup_{
	\vc\H=d}\;\inf_{L\in\L_{\X
	;m,\H}}\,\sup_{D\in\D_{\X
	}}\RR_m(L,D), \\ 
	c_{m,d}^\LB:=\sqrt{
	m/d}\,\inf_\X\inf_{
	\vc\H=d}\;\inf_{L\in\L_{\X
	;m,\H}}\,\sup_{D\in\D_{\X
	}}\RR_m(L,D),
\end{aligned}	
\end{equation}
where $\sup_\X$ and $\inf_\X$ are taken over all finite sets $\X$, 
and $\sup_{\vc\H=d}$ and $\inf_{\vc\H=d}$ are taken over all hypothesis classes $\H$ of VC dimension $d$. 
The quantity 
$\inf_{L\in\L} \sup_{D\in\D}\RR_m(L,D)$ may be referred to as the \emph{minimax 
}EER. 

Integrating both sides of inequality \eqref{eq:upper} in $u\ge0$, one sees that 
\begin{equation}\label{eq:cub<infty}
	c^\UB:=\sup_{m,d}c_{m,d}^\UB
	\le
        \sup_{m,d} \, \sqrt{
        m/d} \,
        \sup_\X\sup_{
        \vc\H=d}\,\sup_
        {D\in\D
        _\X}
	\RR_m(L_\emp,D)<\infty,  
\end{equation} 
where $\sup_{m,d}$ is taken over all natural $m$ and $d$;  
an exact calculation of $c^\UB$
seems to be beyond the reach of current methods. 

It is also clear that inequality \eqref{eq:lower} implies 
\begin{equation}\label{eq:clb>0.0008}
	c_{\nu\ge12.8}^\LB>\frac1{64\sqrt{320}}=0.000873\ldots>0,    
\end{equation}
where
$c_{\nu\ge\nu_*}^\LB:=\inf\{c_{m,d}^\LB\colon m/d\ge\nu_*\}$. 
for any real $\nu_*>0$. 
A
remarkable fact that follows from \eqref{eq:cub<infty} and \eqref{eq:clb>0.0008} is that 
\begin{equation*}
	0<\liminf_{m/d\to\infty}c_{m,d}^\LB\le\limsup_{m/d\to\infty}c_{m,d}^\UB<\infty;  
\end{equation*}
that is, the upper and lower bounds on the minimax 
EER are of the same order of magnitude. 
Establishing an appropriate lower bound on the 
EER, 
$
\RR(L,D)$, was the crucial part of the proof of Theorem~B. 
\newcommand\ul{\underline}


\begin{center}
{
\mycfs{
9.} 
\begin{tabular}{l|l|l}
symbol & brief description & appears in/on  
\\
\hline
$B(m,d)$ & expression for $\inf\limits_{L\in\Lrand} \sup\limits_{D}\RR_m(L,D)$ &\ul{\eqref{eq:IS<}}  \\ 
$B_0(m,d)$ & lower bound on $B(m,d)$ &\ul{\eqref{eq:IS>}}  \\ 
$B_1(\nu)$ & lower bound on $B_0(m,d)$ &\ul{\eqref{eq:IS>hatopt}}  \\ 
$B_2(\nu)$ & lower bound on $B_1(\nu)$ &\ul{\eqref{eq:IS>B_2}}  \\ 
$\tB_2(\nu)$, $\hat B_2(\nu)$ & lower bounds on $B_2(\nu)$ &\ul{\eqref{eq:tc_nu}}, \ul{\eqref{eq:IS>B_2,nu<3}}  \\ 
$\opt(k,b)$ & Bayes risk for $d=1$ &\ul{\eqref{eq:opt:=}}; \eqref{eq:def-opt}
\\ 
\rule{0pt}{10pt}$\hatopt(\ka,b)$ & convex minorant of $\opt(k,b)$ &\ul{Proposition~\ref{lem:convex minor}} \\ 
$b\in[-1,1]$ & $\Y$-bias for $\X=\{1\}$ & symbols $s_k(b)$, 
\dots \\ 
$\g(x)$ & conditional $\Y$-bias at $x\in\X$ &\ul{\eqref{eq:Pbias}} \\ 
$c_{m,d}^\UB$ [$c_{m,d}^\LB$] & $\nu\times$ (exact upper [lower] bound & \\
 & \qquad on the minimax EER) & \ul{\eqref{eq:ub,lb}} \\ 
$c_\infty=0.16997\dots$ & limit value of $c_{m,d}^\LB$ & \eqref{eq:clb->c_infty} \\ 
$c_\nu$, $\tc_\nu$ & close lower bounds on $c_{m,d}^\LB$ & \ul{\eqref{eq:c_nu}}, \ul{\eqref{eq:tc_nu}} \\ 
$C_i$ &  & \ul{\eqref{eq:c_nu}}, \eqref{eq:C_i} \\ 
$D$, $D_{p,\g}$ & distribution on $\X\times\Y$ &\ul{p.\ \pageref{p_intro}}, p.\ \pageref{p,g} \\ 
$d:=\vc\H$ & VC dimension of $\H$ &\ul{p.\ \pageref{VC}} \\ 
$\De(h,D)$ & excess risk & \ul{\eqref{eq:Delta}}  \\ 
$\err(h,D)$ &  error probability &\ul{\eqref{eq:err-def}} \\ 
$\err_{\min}(h,D)$ & minimum error probability &\ul{\eqref{eq:min-err-def}} \\ 
$h\colon\X\to\Y$ & hypothesis &\ul{p.\ \pageref{p_intro}} \\ 
$\H$ & hypothesis class &\ul{p.\ \pageref{p_intro}}, p.\ \pageref{[d]} \\ 
$\ii{\cdot}$ & indicator function &\ul{below \eqref{eq:err-decomp}} \\  
$L$ & learning algorithm (l.a.) &\ul{p.\ \pageref{L}} \\ 
$L_\emp$, $L^*_\erm$ & empirical risk minimizer &\ul{p.\ \pageref{L_emp}}, \eqref{eq:L_erm}, \eqref{eq:L*_erm} \\ 
$\L$ & set of all non-randomized l.a.'s &\ul{p.\ \pageref{LL}}, p.\ \pageref{Lrand} \\ 
$\Lrand$ & set of all randomized l.a.'s &\ul{p.\ \pageref{Lrand}} \\ 
$m$ & labeled sample size &\ul{p.\ \pageref{p_intro}} \\ 
$N$ & binomial r.v.\ w/ parameters $m$, $1/d$ &\ul{Theorem~\ref{th:B_0}}  \\ 
$N^p_x$ & cardinality of the set $\{i\colon X_i=x\}$ &\ul{Theorem~\ref{th:duality}}; \eqref{eq:N_x} \\ 
$\nu:=m/d$ & fundamental ratio &\ul{\eqref{eq:nu:=}} \\ 
$p$ & $\X$-marginal of  $D$ &\ul{\eqref{eq:Pbias}} \\ 
$\RR
(L,D)$, $\RR(L;p,\g)$ & expected excess risk (EER) & \ul{\eqref{eq:R}}, \eqref{eq:E De_emp} \\ 
$\sign$ & modified sign function &\ul{below \eqref{eq:err-decomp}} \\ 
$s_k(b)$ &  &\ul{\eqref{eq:s_k:=}} \\ 
$V^b_k$, $Y_i^{b}$ &  &\ul{\eqref{eq:V^b_k}} \\ 
$V^{p,\g}_x$ & ``vote balance'' at $x\in\X$ &\ul{\eqref{eq:V_x}} \\ 
$(X^{D}_i,Y_i^{D})$, $(X^p_i,Y_i^{p,\g})$ & labeled sample items &\ul{p.\ \pageref{p_intro}}, p.\ \pageref{p,g} \\ 
$\X$ &set of objects to classify &\ul{p.\ \pageref{p_intro}}, p.\ \pageref{[d]} \\ 
$\Y=\{-1,1\}$ &set of classification labels &\ul{p.\ \pageref{p_intro}} \\ 
$Z^{D}_m$, $Z^{p,\g}_m$ & labeled sample &\ul{\eqref{eq:Z_m}}, \eqref{eq:Z_m^p,g} \\ 
$z_*=0.75179\dots$ & maximizer of $\frac{z}2\,\big(1-\erf(z/\sqrt2)\big)$ &\ul{\eqref{eq:z*}}  \\ 
\end{tabular}
}
\label{tab:notations}
\captionof{table}{Notations. The places where the symbols are first introduced are underlined.}
\end{center}


\newpage

A few words on 
the organization of the rest of this paper: The main results are stated and discussed in Section~\ref{results}. All necessary proofs are given in Section~\ref{proofs}, with more technical parts deferred further, to Appendices~\ref{app:bayes}
--\ref{c^LB reduction}.

An index of symbols used in this paper non-locally is given in Table~
1, 
which lists the places where the selected symbols are first introduced and, for a few of the symbols, the places where those symbols are generalized, specialized, or otherwise modified.  

\section{Results: statements and discussion}\label{results} 

In this paper, we present optimal lower bounds on the minimax 
EER, 
which cannot be further improved. 
Our main result is Theorem~\ref{th:duality}, which provides an expression of the exact, non-asymptotic lower bound on the minimax EER. 
This expression is in terms of a certain function $\opt(k,b)$, which is the Bayes risk for $d=1$. We show (in Proposition~\ref{lem:convex minor}) that $\opt(k,b)$ has a certain convexity property with respect to $k$. Further important properties of the function $\opt$, based on certain apparently novel identities and inequalities for binomial distributions, are presented in Appendix~\ref{app:bayes}.  
Thus, the expression of the non-asymptotic lower bound on the minimax EER given in Theorem~\ref{th:duality} becomes amenable to high-precision analysis. 
\big(Implicitly, the function $\opt$ is present in
\citet{
BK}, but there it was bounded via
Pinsker's inequality, which yields sub-optimal results.%
\big)

In particular, 
based on Theorem~\ref{th:duality} and 
the mentioned analysis of the function $\opt$, we 
determine (in Theorem~\ref{th:negl}) the asymptotics of the just mentioned exact lower bound: 
\begin{equation}\label{eq:clb->c_infty}
	c_{m,d}^\LB\to c_\infty:=\max_{z>0}\tfrac{z}2\,\big(1-\erf(z/\sqrt2)\big)=0.16997\dots   
\end{equation}
whenever $m$ and $d$ vary in such a way that $
\nu=m/d\to\infty$;  
here, as usual, $\erf$ denotes the Gauss error function, given by the formula $\erf(u):=\frac2{\sqrt\pi}\,\int_0^u e^{-t^2}dt$ for real $u\ge0$. 

It should be noted that in Theorem~\ref{th:duality} 
randomization of learning algorithms is allowed; however, it will also be shown (in Theorem~\ref{th:negl})
that the effect of this randomization is asymptotically negligible and is entirely explained by ties in a
certain ``voting'' procedure. 

Theorems~\ref{th:B_0}, \ref{th:B_1}, \ref{th:B_2,tB_2}, and Proposition~\ref{prop:B_2,B_3} present, for
finite $m$ and $d$, tractable lower bounds on $c_{m,d}^\LB$; 
Theorem~\ref{th:asymp} then shows that all these lower bounds on $c_{m,d}^\LB$, as well as $c_{m,d}^\LB$ itself,
converge to the limit constant $c_\infty$ in \eqref{eq:clb->c_infty}. 
%
%
Moreover, it is shown (see Remarks~\ref{rem:B_0} and \ref{rem:B_2,B_3}, and Figure~\ref{fig:}) that these lower bounds on $c_{m,d}^\LB$, as well as $c_{m,d}^\LB$ itself, are actually close to the limit value $c_\infty$ even for rather small values of $\nu=m/d$. 

The above discussion suggests a sense of completion in the area of lower bounds for the PAC model. However, results and techniques presented here may be used elsewhere. In fact, they 
%
already found an application in
\citet[Theorem 7.1]{
KontorovichSU16}, \label{SU16}
where existing lower bounds were not sufficiently delicate
for the desired parameter regime.


In this paper, we apply our lower bounds on the EER to obtain substantial improvements of the existing lower bounds on the tail probability of the excess risk, as follows:  

\begin{theorem}\label{th:P low}\ 
\begin{enumerate}[(i)]
	\item Keeping the constants $12.8$ and $320$ in Theorem~B in place, one can improve the lower bound $\frac1{64}\approx0.0156$ on the tail probability in \eqref{eq:lower} to $0.238$. 
	\item Keeping the constants $12.8$ and $\frac1{64}$ in Theorem~B in place, one can improve the constant $320$ in \eqref{eq:lower} to $41.3$. 
	\item If the restriction $\nu\ge12.8$ in Theorem~B is relaxed to $\nu\ge3$, then the improved values $0.238$ and $41.3$ of the constants get only slightly worse: $0.227$ and $49.6$, respectively. 
\end{enumerate}
\end{theorem}

\hrule
\medskip

To state our results, let us introduce some additional notation and conventions to be used in 
the sequel. 

Let $0^0:=1$. 

For any $\al$ and $\omega$ in $\ZZ\cup\{\infty\}$, let
$\intr\al\omega:=\{i\in\ZZ\colon\al\le i\le\omega\}$. For any
$m\in\intr0\infty$, let $[m]:=\intr1m$. In particular, $[0]=\emptyset$. 

As usual, for any two sets $S$ and $T$, let $S^T$ denote the set of all maps from $T$ to $S$. 

For any set $A$ and any  $k\in\intr0\infty$ we identify the $k$-tuples
$v=(v_1,\dots,v_k)\in A^k$ with functions $v\colon[k]\to A$, by the formula ${v}(x):={v}_x$ for all $x\in[k]$;  
thus, we identify the set $A^k$ of $k$-tuples with the set $A^{[k]}$ of functions. 
So, we use notations $v(x)$ and $v_x$ interchangeably. We shall also identify a function with its graph. 

As usual, the sum of the empty family of elements of a linear space is defined as the zero element of that space.

The new results obtained in this paper all concern the lower bound $c_{m,d}^\LB$, defined in \eqref{eq:ub,lb}, on the minimax 
EER times the factor $\sqrt{
m/d}$, including the limit behavior of $c_{m,d}^\LB$ as $
m/d\to\infty$. 

It is not hard to show (see Appendix~\ref{c^LB reduction} for details) that  
the defining expression for $c_{m,d}^\LB$ in \eqref{eq:ub,lb} can be simplified as follows: 
\begin{equation}\label{eq:clb}
	c_{m,d}^\LB=\sqrt{
	m/d}\,\inf_L\sup_D
	\RR_m(L,D),
\end{equation} 
where from now on it will be assumed (unless otherwise specified) that 
\\ 
\noindent
{
\setlength{\fboxsep}{-2pt}
\framebox{\parbox{\textwidth}{
\begin{equation*}
	\X=[d],\quad \H=\Y^\X={\{-1,1\}}^{[d]},\quad 
	\inf_L:=\inf_{L\in\L_{[d]
	;m,\Y^{[d]}}},\quad \sup_D:=\sup_{D\in\D_{[d]
	}}, 
\end{equation*}
\label{[d]}
}}
}

\noindent so that 
$\inf_L$ is taken over all learning algorithms $L\colon([d]\times\Y)^m\to\Y^{[d]}
$ and $\sup_D$ is taken over all distributions $D$ on $[d]\times\Y$. 
Accordingly, from now on we shall use $\X$ and $\H$ interchangeably with $[d]$ and $\Y^\X={\{-1,1\}}^{[d]}$, respectively.

Note next that any distribution $D$ on $\X\times\Y$ is completely characterized by the two maps, say $p=p_D\colon\X\to[0,1]$ and $\g=\g_D\colon\X\to[-1,1]$, such that 
\begin{equation}\label{eq:Pbias}
	\P(X^D=x,Y^D=y)=D(\{(x,y)\})
	=p(x)\Big(\frac12+\frac{y\g(x)}2\Big) 
\end{equation}
for $x\in\X=[d]$ and $y\in\Y=\{-1,1\}$. 
Clearly then, one must have $p_D(x)=\P(X^D=x)$ for all $x\in\X$ and $\g_D(x)=2
\P(Y^D=1|X^D=x)-1$ for all $x\in\X$ with $p_D(x)\ne0$; if $p_D(x)=0$ for some $x\in\X$, then the value of $\g_D(x)$ can be chosen arbitrarily in $[-1,1]$, So,  
the distribution of the random variable (r.v.) $X^D$ is completely characterized by the map $p=p_D$.  
\label{p,g}
Therefore, in what follows let us write $D=D_{p,\g}$ if $p_D=p$ and $\g_D=\g$, and, in the case when $D=D_{p,\g}$, let us 
simply write $X^p$, $Y^{p,\g}$, $X_i^p$, $Y_i^{p,\g}$ instead of $X^D$, $Y^D$, $X_i^D$, $Y_i^D$ (respectively), assuming that the random pairs $(X_1^p,Y_1^{p,\g}),\dots,(X_m^p,Y_m^{p,\g})$ are independent copies of the random pair $(X^p,Y^{p,\g})=(X^D,Y^D)$;  
let us then also write 
\begin{equation}\label{eq:Z_m^p,g}
Z_m^{p,\g}:=Z_m^D:=\big((X_1^p,Y_1^{p,\g}),\dots,(X_m^p,Y_m^{p,\g})\big) 
\end{equation}
(cf.\ \eqref{eq:Z_m}). 

Take next any $h\in\H=\{-1,1\}^{[d]}$. 
It is well known (see e.g. \cite[page~10]{DGL}) that the function $\hopt\in\{-1,1\}^{[d]}$ given by the formula 
\begin{equation}\label{eq:h_ga}
	\hopt(x):=\sign\g_x
\end{equation}
for $x\in[d]$ is a minimizer of $\err(h,D_{p,\g})$ over all $h\in\{-1,1\}^{[d]}$, and the excess risk (relative to $D_{p,\g}$) of $h$ over $\hopt$ is 
\begin{align}
\Delta(h,D_{p,\g}) 
= \err(h,D_{p,\g})-\err(\hopt,D_{p,\g}) 
=\sum_{x=1}^d p_x|\g_x|\ii{h(x)\ne\hopt(x)},       
\label{eq:err-decomp}
\end{align}
where 
$$\sign u:=2\ii{u\ge0}-1$$ 
for real $u$ and $\ii{\cdot}$ is the indicator function.   

%

Replacing now the unknown true distribution $D=D_{p,\g}$ by the empirical distribution $\hat D_m=\hat D_m\big((x_1,y_1),\dots,(x_m,y_m)\big)$ for $\big((x_1,y_1),\dots,(x_m,y_m)\big)=:z\in(\X\times\Y)^m$, one sees that a function $h\in\{-1,1\}^{[d]}$ is a minimizer of $\err(h,\hat D_m)$ for the given ``sample'' $z$ if and only if $h(x)=\sign\widehat{p\g}_x$ for all $x\in\X$ such that $\widehat{p\g}_x\ne0$, where 
$
	\widehat{p\g}_x:=\frac1m\,\sum_{i=1}^m y_i\ii{x_i=x};  
$ 
if $\widehat{p\g}_x=0$ for some $x\in\X$, then the value $h(x)$ of a minimizer $h$ (of $\err(h,\hat D_m)$) at this point 
$x$ can be chosen arbitrarily in the set $\{-1,1\}$. 
Thus, all the learning algorithms $L_\erm$ that are minimizers of the empirical risk are given by the formula 
\begin{equation}\label{eq:L_erm}
	L_\erm(z_m)(x):=L_{m,d;\,\erm}
	(z_m)(x)  
	\left\{
	\begin{alignedat}{2}
	&:=\sign v_x &&\text{\ \ if\ \ }v_x\ne0, \\
	&\in\{-1,1\} &&\text{\ \ if\ \ }v_x=0  
	\end{alignedat}
	\right.
\end{equation}
for all $z_m\in(\X\times\Y)^m$ and $x\in\X=[d]$, where 
\begin{equation}\label{eq:v_x}
	v_x:=v_x(z_m):=\sum_{i=1}^m y_i\ii{x_i=x}
	=m\,\widehat{p\g}_x. 
\end{equation} 
Formula \eqref{eq:L_erm} states that the empirical risk is minimized when the 
value $y\in\{-1,1\}$ assigned by the learning algorithm at point $x$ based on the ``sample'' $z_m$ is decided by the majority vote $v_x=v_x(z_m)$ ``at $x$'', with the ``voting'' restricted to the pairs $(x_i,y_i)$ with $x_i=x$; if there is a tie (no majority) at $x$, then a value $y\in\{-1,1\}$ at $x$ is chosen arbitrarily. 

To decrease the risk and also be able to fully use the power of decision theory, one may randomize learning algorithms
. 
\label{rand}
A convenient way to define 
such an algorithm $L$ is 
to allow its value (which is a function in $\H$) to depend, not only on the nonrandom ``sample'' $z_m=\big((x_1,y_1),\dots,(x_m,y_m)\big)\in\break
(\X\times\Y)^m$ as in \eqref{eq:z_m}, but also on the value $u$ of another r.v., say $U$, which is (say) uniformly distributed on the interval $[-1,1]$ and independent of the random ``sample'' $Z_m^D=\big((X^{D}_1,Y_1^{D}),\dots,(X_m^{D},Y_m^{D})\big)$ as in \eqref{eq:Z_m}. Thus, a randomized learning algorithm $L$ will be understood as a 
map from $(\X\times\Y)^m\times[-1,1]$ to  $\H$. 

\label{Lrand}
Let $\Lrand=\Lrandm 
d$ and $\L=\L_
d$ denote, respectively, the set of all randomized learning algorithms and the set of all non-randomized ones. 
The definition \eqref{eq:R} of the 
EER (for $L\in\L$) is naturally extended as follows: 
\begin{equation}\label{eq:R,U}
	\RR(L,D):=\RR_m(L,D):=\E\De\big(L(Z_m^D,U),D\big) 
\end{equation}
for $L\in\Lrand$;  
it then follows by \eqref{eq:err-decomp} that 
\begin{align}
\RR_m(L;p,\g):=\RR_m(L,D_{p,\g})=\sum_{x=1}^d p_x|\g_x|\P\big(L(Z_m^{p,\g},U)(x)\ne\hopt(x)\big).    
\label{eq:E De_emp} 
\end{align}

Of particular importance will be the following ``maximally symmetric'' and ``minimally randomized'' version of the learning algorithms $L_\erm$ that are minimizers of the empirical risk (cf.\ \eqref{eq:L_erm}): 
\begin{equation}\label{eq:L*_erm}
	L^*_\erm(z_m,u)(x):= 
	L^*_{d,\erm}(z_m,u)(x):=
	\left\{
	\begin{alignedat}{2}
	&\sign v_x &&\text{\ \ if\ \ }v_x\ne0, \\
	&y_{
	i_x} &&\text{\ \ if\ \ }v_x=0\text{ but }n_x\ne0, \\
	&\sign u &&\text{\ \ if\ \ }n_x=0   
	\end{alignedat}
	\right.
\end{equation}  
for $(z_m,u)\in(\X\times\Y)^m\times[-1,1]=([d]\times\{-1,1\})^m\times[-1,1]$, 
where 
\begin{equation}\label{eq:n_x,i_x}
	n_x:=n_x(z_m):=\sum_{i=1}^m \ii{x_i=x}\quad\text{and}\quad
	i_x:=i_x(z_m):=\min\{i\in[m]\colon x_i=x\}. 
\end{equation}
That is, the choice of the value of $L^*_\erm(z_m,u)(x)$ in $\Y=\{-1,1\}$ is decided by the majority vote ``at $x$'' if there is a majority there; otherwise, the value $L^*_\erm(z_m,u)(x)$ is the same as that of the first voter that appeared ``at $x$'' if any one did; finally, if no one arrived to vote ``at $x$'', then the value is decided by a flip of a fair coin, the flip being independent of any voters. Thus, randomization according to the learning algorithm $L^*_\erm$ occurs only if no one shows up for voting at some location $x\in\X$. 
Yet, this minimal (and, one may argue, quite natural) randomization is enough to make $L^*_\erm$ a winner (that is, a minimax learning algorithm) against all randomized (and non-randomized) learning algorithms. A precise formulation of this thesis is contained in  

\begin{theorem}\label{th:duality} 
Take any $m\in\intr0\infty$. Then 
\begin{multline}\label{eq:IS<}
	\inf_{L\in\Lrand} \sup_{D}\RR_m(L,D)
	=\sup_{D}\RR_m(L^*_\erm,D) \\ 
	=B(m,d)  
	:=\sup_{p,\g}\sum_{x=1}^d p_x\,|\g_x|\,\E\opt(N^p_x,|\g_x|),  
\end{multline} 
where   
$\sup_{p,\g}$ is taken over all pairs of functions $p\in[0,1]^{[d]}$ such that $\sum_{x=1}^d p_x=1$ and $\g\in[-1,1]^{[d]}$,  
$N^p_x$ is a r.v.\ with the binomial distribution with parameters $m$ and $p_x$, 
\begin{equation}\label{eq:opt:=}
	\opt(k,b):=\tfrac12\,\big(1-s_k(b)\big), 
\end{equation}
\begin{equation}\label{eq:s_k:=}
	s_k(b):=|\P(V^b_k>0)-\P(V^{-b}_k>0)|, 
\end{equation}
\begin{equation}\label{eq:V^b_k}
	V^b_k:=Y_1^{b}+\dots+Y_k^{b}, 
\end{equation}
and the $Y_i^{b}$'s are iid 
r.v.'s with  
$\P(Y_i^{b}=1)=\frac{1+b}2$ and $\P(Y_i^{b}=-1)=\frac{1-b}2[=1-\P(Y_i^{b}=1)]$, for $k\in\intr0\infty$ and $b\in[-1,1]$.   
Moreover, for each pair of functions $p$ and $\g$ as described above, 
\begin{equation}\label{eq:RR*}
\RR_m(L^*_\erm;p,\g)=
	\RR_m(L^*_\erm,D_{p,\g})=\sum_{x=1}^d p_x\,|\g_x|\,\E\opt(N^p_x,|\g_x|), 
\end{equation}
which does not depend on $\sign\g:=(\sign\g_1,\dots,\sign\g_d)$. 
\end{theorem} 

The use of the symbol $\opt$ in \eqref{eq:IS<} is a reflection of the fact that the minimax learning algorithm $L^*_\erm$ is a Bayes one with respect to a certain prior distribution on the set of all distributions $D$ on $\X\times\Y$; see the beginning of the proof of Theorem~\ref{th:duality} in Section~\ref{proofs} for details on this. 
Formula \eqref{eq:RR*} means that the learning algorithm $L^*_\erm$ has an important risk-equalizing property --- which actually makes the Bayes decision rule $L^*_\erm$ minimax; cf.\ e.g.\ Theorem~3 and Lemma~1 in \cite[\S2.11]{ferguson62_book}. 

\begin{remark}\label{rem:upper}
It is clear from \eqref{eq:opt:=}--\eqref{eq:s_k:=} that $\opt\le\frac12$. Hence, by \eqref{eq:IS<}, $\inf_{L\in\Lrand} \sup_{D}\RR_m(L,D)\le\frac12$. 
\end{remark}

It turns out, as may be expected, that the effect of the randomization of learning algorithms is asymptotically negligible whenever $\nu=m/d\to\infty$; that is, the difference $\inf_{L\in\L} \sup_{D}-\inf_{L\in\Lrand} \sup_{D}$ is asymptotically negligible compared with the 
``non-random\-ized'' minimax 
EER $\inf_{L\in\L} \sup_{D}$. 
Moreover, all the learning algorithms of the form $L_\erm$ as in \eqref{eq:L_erm} that are minimizers of the empirical risk are asymptotically minimax. 
These facts --- along with the asymptotics of the minimax risk --- are presented in  

\begin{theorem}\label{th:negl} 
For each pair $(m,d)$ of natural numbers, choose any learning algorithm of the form $L_{m,d;\erm}$, as in \eqref{eq:L_erm}. Then 
\begin{multline}\label{eq:IS sim}
	\frac{c_\infty}{\sqrt{
	m/d}}\sim\inf_{L\in\Lrand} \sup_{D}\RR_m(L,D)
	\le\inf_{L\in\L} \sup_{D}\RR_m(L,D) \\ 
	\le \sup_{D}\RR_m(L_{m,d;\erm},D)  
	\sim \frac{c_\infty}{\sqrt{
	m/d}} 
\end{multline} 
whenever $
m/d\to\infty$, where $c_\infty=0.16997\dots$ as in \eqref{eq:clb->c_infty}. 
Moreover, 
\begin{multline}\label{eq:ties}
	0\le\sup_{D}\RR_m(L_{m,d;\erm},D)-\inf_{L\in\Lrand} \sup_{D}\RR_m(L,D) \\ 
	\le\frac12\sup_{p,\g}\,\sum_{x=1}^d p_x|\g_x|\,
	\P\big(V^{p,\g}_x=0\big)
	=O\Big(\frac1{
	m/d}\Big)  
	=o\Big(\frac1{\sqrt{
	m/d}}\Big),  
\end{multline}
again whenever $
m/d\to\infty$, where 
\begin{equation}\label{eq:V_x}
	V^{p,\g}_x:=\sum_{i=1}^m Y^{p,\g}_i\ii{X^p_i=x}, 
\end{equation}
the vote ``balance'' at $x$ based on the random ``sample''  $Z_m^{p,\g}$ 
as in \eqref{eq:Z_m^p,g}. 
\end{theorem}

Here, as usual, the asymptotic equivalence $A\sim B$ means $A/B\to1$. 

Display \eqref{eq:ties} shows that the (asymptotically negligible) pairwise differences between (i) the minimax 
EER $
\inf_{L\in\L} \sup_{D}\RR_m(L,D)$, (ii) its ``randomized'' version $
\inf_{L\in\Lrand} \sup_{D}\RR_m(L,D)$, and (iii) the maximum risk $
\sup_{D}\RR_m(L_{m,d;\erm},D)$ of any empirical-risk-minimizing learning algorithms of the form $L_\erm$ are entirely explained by ties in the mentioned ``voting'', when the ``no-majority'' event $V^{p,\g}_x=0$ occurs for at least one $x\in\X=[d]$. 

It is 
obvious from \eqref{eq:IS<} 
that 
\begin{equation*}
	B(m,d)
\ge\sum_{x=1}^d \frac1d\,b\,\E\opt(N_x,b)
=\,b\,\E\opt(N_1,b)
\end{equation*}
for any $b\in[0,1]$, where $N_x$ stands for $N^p_x$ with $p_x=\frac1d$ for all $x\in[d]$. 
Thus, in view of \eqref{eq:IS sim}, one immediately obtains 

\begin{theorem}\label{th:B_0}  
\begin{equation}\label{eq:IS>}
	\inf_{L\in\L}\sup_{D}\RR_m(L,D)\ge \inf_{L\in\Lrand}\sup_{D}\RR_m(L,D)\ge B_0(m,d):=\sup_{b\in[0,1]} b\,\E\opt(N,b),  
\end{equation} 
where $N$ is
a binomial r.v.
with parameters $m$ and $1/d$. 
\end{theorem}

Recall \eqref{eq:Z_m^p,g} and let 
\begin{equation}\label{eq:D_ga}
	D_\g:=D_{p,\g}\quad\text{and}\quad Z_m^\g:=Z_m^{p,\g}\quad\text{when $p_x=\tfrac1d$ for all $x\in\X=[d]$}. 
\end{equation}

\begin{theorem}\label{th:average} 
For any $b\in[0,1]$, 
\begin{equation}\label{eq:aver>hatopt}
  \inf_{L\in\Lrand} 
  \;\frac1{2^d}\sum_{\g\in\{-b,b\}^{[d]}}
  \RR(L,D_\g)= b\,\E\opt(N,b).    
\end{equation}
\end{theorem}

As we shall see, Theorem~\ref{th:average} follows immediately from the proof of Theorem~\ref{th:duality}. 
On the other hand, Theorem~\ref{th:average} could be viewed 
as 
a refinement of Theorem~\ref{th:B_0}, because clearly 
$
\frac1{2^d}\sum_{\g\in\{-b,b\}^{[d]}}
\RR(L,D_\g)\le\sup_{D}
\RR(L,D)$ for any $L$. 
Even though 
the refinement is slight, 
Theorem~\ref{th:average} will be useful, in particular, in the proof of Theorem~\ref{th:P low}.

\begin{remark}\label{rem:B_0}
Note that,  
by 
\eqref{eq:opt:=}--\eqref{eq:s_k:=}, $\opt(k,b)$ is a polynomial in $b$ of degree $\le k$. Hence, $b\,\E\opt(N,b)$ is a polynomial in $b$ of degree $\le m+1$, and so, the lower bound $B_0(m,d)$ in \eqref{eq:IS>} is an algebraic number, which is not hard to compute unless $m$ is too large. 
For instance, for $c_0(m,d):=B_0(m,d)\sqrt{m/d}$ we find 
\begin{equation}\label{eq:c_0(m,2)}
c_0(5,2)=0.16757\dots,\quad c_0(50,20)=0.17467\dots,\quad\text{and}\quad c_0(50,2)=0.16968\dots 	
\end{equation}
(with the execution times in Mathematica about $0.02$ sec, $1.4$ sec, and $1$ sec, respectively). 
One may note that, even for such a rather small value $5/2=50/20=2.5$ 
of $\nu=m/d$, the values of $c_0(m,d)$ are close to the limit value $c_\infty=0.16997\dots$ --- cf.\ \eqref{eq:clb->c_infty} and \eqref{eq:IS sim}. 
However, more work needs to be done to more fully understand the manner in which the lower bound $B_0(m,d)$ depends on $m$ and $d$. 
\end{remark}

The important first step toward this goal is establishing the following convexity property of the function $k\mapsto\opt(k,b)$. 

\begin{proposition}\label{lem:convex minor}
Take any $b\in[0,1]$. Then the largest convex function $[0,\infty)\ni\ka\mapsto\hatopt(\ka,b)$ such that $\hatopt(k,b)\le\opt(k,b)$ for all $k\in\{0,1,\dots\}$ is given by the formula 
\begin{multline}\label{eq:hatopt:=}
	\hatopt(\ka,b)\\ 
	:= \left\{
	\begin{alignedat}{2}
	&(1-\ka)\opt(0,b)+\ka\opt(1,b)=\tfrac12\,(1-\ka b)&&\text{\ \;if\ \;}0\le\ka\le1, \\ 
	&\tfrac{2i+3-\ka}2\,\opt(2i+1,b)+\tfrac{\ka-2i-1}2\,\opt(2i+3,b)&&\text{\ \;if\ \;}2i+1\le\ka\le2i+3  
	\end{alignedat}
	\right.
\end{multline}
for any $i\in\intr0\infty$. 
\end{proposition}

That is, the largest convex minorant $\hatopt(\cdot,b)$ on $[0,\infty)$ of the function $\opt(\cdot,b)$ on $\{0,1,\dots\}$ is just the linear interpolation of $\opt(\cdot,b)$ at the points $0,1,3,5,\dots$. 
This is illustrated in Figure~\ref{fig:hatbayes}.  

\begin{figure}[htbp]
	\centering
		\includegraphics[width=.60\textwidth]{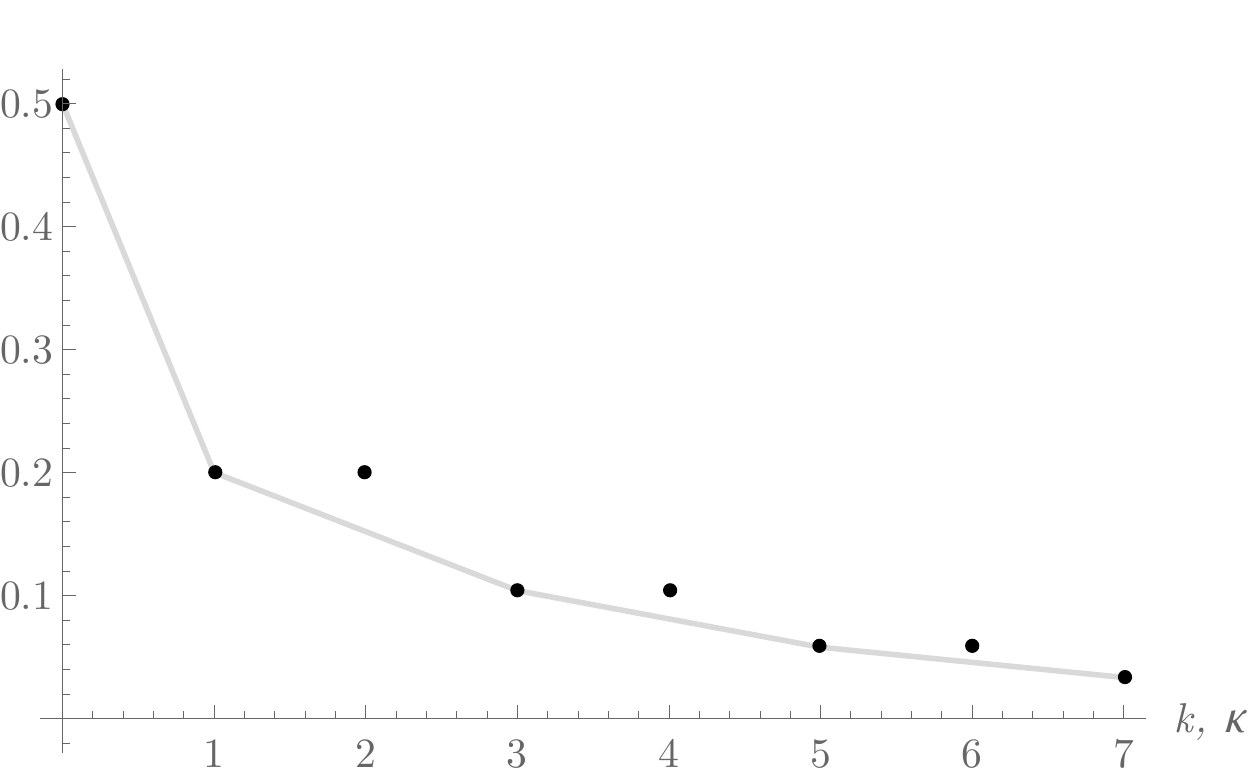}
	\caption{
Graphs of the maps $\{0,1,\dots,7\}\ni k\mapsto\opt(k,b)$ (black dots) and $[0,7]\ni\ka\mapsto\hatopt(\ka,b)$ (gray broken line) for $b=0.6$.  
	}
	\label{fig:hatbayes}
\end{figure}

Recall the definition \eqref{eq:nu:=} of $\nu$. 
Using \eqref{eq:IS>}, Proposition~\ref{lem:convex minor}, Jensen's inequality, and the equality $\E N=\nu$, one immediately obtains

\begin{theorem}\label{th:B_1} 
\begin{equation}\label{eq:IS>hatopt}
  \inf_{L} \sup_{D}
  \RR(L,D)\ge B_0(m,d)\ge B_1(\nu):=
	\sup_{b\in(0,1)}b\,\hatopt(\nu,b).    
\end{equation}
\end{theorem}
Here and in the rest of this section, $\inf_{L}$ can 
be replaced by either $\inf_{L\in\L}$ or $\inf_{L\in\Lrand}$.

\begin{remark}\label{rem:B_1}
An advantage of the lower bound $B_1(\nu)$ in \eqref{eq:IS>hatopt} over the bound $B_0(m,d)$ in \eqref{eq:IS>} is that it depends only on $\nu=m/d$; also, $B_1(\nu)$ is not hard to compute unless $\nu$ is too large. 
Yet, the nature of the dependence of $B_1(\nu)$ on $\nu$ may still seem rather obscure. 
Therefore, we are going to present a lower bound on $B_1(\nu)$ that is much easier to grasp and yet is (i) asymptotic to the original lower bound $B_0(m,d)$ for $\nu=m/d\to\infty$ and (ii) 
close to $B_0(m,d)$ even for moderate values of $\nu=m/d$. 
\end{remark}

In Appendix~\ref{app:bayes}, we shall obtain explicit and rather tight lower bounds on the function $\opt$. In view of Theorem~\ref{th:B_1} and Proposition~\ref{lem:convex minor}, this will result in explicit lower bounds on the minimax excess risk 
$\inf_{L} \sup_{D}
\RR(L,D)$, as follows. 

Let 
\begin{equation}
  \label{eq:z*}
	z_*=0.75179\dots
\end{equation}
be the unique maximizer of $\frac{z}2\,\big(1-\erf(z/\sqrt2)\big)$ in real $z>0$, with the maximum value $c_\infty=0.16997\dots$, as in \eqref{eq:clb->c_infty}. 

\begin{theorem}\label{th:B_2,tB_2}
Assume that 
$\nu\ge 
1$. Let $i_\nu:=\lfloor\frac{\nu-1}2\rfloor$. 
Then 
\begin{equation}\label{eq:IS>B_2}
	\inf_{L} \sup_{D}
	\RR(L,D)\ge B_1(\nu)\ge
	B_2(\nu):=\frac{c_\nu}{\sqrt\nu},   
\end{equation}
where 
\begin{equation}\label{eq:c_nu}
	c_\nu:=\frac{z_*}2\,\Big(1-C_{i_\nu}\frac{\erf(z_*/\sqrt2)}{
	\exp\{-z_*^2/(6\nu)\}}\Big)<c_\infty  
\quad\text{and}\quad 
C_i=\frac{\sqrt{\pi(i+1/2)}}{2^{2i}}\,\binom{2i}i	
\end{equation} 
for $i=0,1,\dots$. 
Moreover, for $\nu\ge3$, $B_2(\nu)$ admits a simple lower bound on it: 
\begin{equation}\label{eq:tc_nu}
B_2(\nu)\ge\tB_2(\nu):=\frac{\tc_\nu}{\sqrt\nu},\quad\text{where}\quad 
	\tc_\nu:=\frac{z_*}2\,
	\bigg(1-\Big(\frac{i_\nu +1}{i_\nu}\Big)^{1/8}\frac{\erf(z_*/\sqrt2)}{
	\exp\{-z_*^2/(6\nu)\}}\bigg)\le c_\nu. 
\end{equation}
\end{theorem}

\begin{remark}\label{rem:cheb,aver}
To obtain the second inequality in \eqref{eq:IS>B_2} \big($B_1(\nu)\ge B_2(\nu)$\big), 
in the proof of Theorem~\ref{th:B_2,tB_2} we are going to use, in particular, two facts: (i) that $C_i$ decreases in $i$ (as stated in Lemma~\ref{lem:C_i}) and (ii) the concavity of $\erf(b\sqrt{k/2})$ in $k$. If one also uses the obvious fact that $\erf(b\sqrt{k/2})$ increases in $k$, then, by Chebyshev's integral inequality, 
\begin{align*}
	&(1-w_i)C_i\erf\big(b\sqrt{\tfrac{2i+1}2}\big)+w_i C_{i+1}\erf\big(b\sqrt{\tfrac{2i+3}2}\big) \\ 
	&\le\big[(1-w_i)C_i+w_i C_{i+1}\big]\,
	\big[(1-w_i)\erf\big(b\sqrt{\tfrac{2i+1}2}\big)+w_i \erf\big(b\sqrt{\tfrac{2i+3}2}\big)\big] \\ 
	&\le\big[(1-w_i)C_i+w_i C_{i+1}\big]\,\erf(b\sqrt{\nu}), 
\end{align*}
where $i:=i_\nu$ and $w_i:=\tfrac{\nu-2i-1}2\in[0,1)$. 
Thus, one can replace $C_{i_\nu}=C_{i_\nu}\vee C_{i_\nu+1}$ in \eqref{eq:c_nu} by the smaller (and hence better) value $(1-w_i)C_i+w_i C_{i+1}$, with $i=i_\nu$.  
Quite similarly, one can replace $\tC_{i_\nu}:=\big(\frac{i_\nu +1}{i_\nu}\big)^{1/8}=\tC_{i_\nu}\vee \tC_{i_\nu+1}$ in \eqref{eq:tc_nu} by the smaller (and hence better) value $(1-w_i)\tC_i+w_i \tC_{i+1}$, with $i=i_\nu$. 
However, these improvements are comparatively small, especially for larger values of $\nu$, and the resulting expressions will be less easy to perceive.
\end{remark}

It is clear that 
\begin{equation}\label{eq:c_nu,tc_nu}
	c_\nu\to c_\infty\quad\text{and}\quad\tc_\nu\to c_\infty
\end{equation}
as $\nu\to\infty$. In fact, $c_\nu$ and even $\tc_\nu$ are rather close to $c_\infty$ already for rather small values of $\nu$. E.g., one has $c_5=0.1553
6\dots$, $\tc_5=0.1551
4\dots$, $c_{50}=0.16852\dots$, and $\tc_{50}=0.16852\dots$, and indeed all these four values are rather close to $c_\infty=0.16997\dots$. We also see that the values of $\tc_\nu$ are not only simpler to compute than, but also very close to, the corresponding values of $c_\nu$. 
 
Inequality \eqref{eq:IS>B_2} in Theorem~\ref{th:B_2,tB_2} does not cover the case $0<\nu<1$, and inequality \eqref{eq:tc_nu} 
does not cover the case $0
<\nu<3$. These two apparently less important cases are covered, complementarily, by 

\begin{proposition}\label{prop:B_2,B_3}
\begin{multline}\label{eq:IS>B_2,nu<3}
	\inf_{L} \sup_{D}
	\RR(L,D)
	\ge \hat B_2(\nu):= \\ 
	\left\{
	\begin{alignedat}{2}
	  &B_1(\nu)=
	  \tfrac12\,(1-\nu) &&\text{\ \;if\ \;}0<\nu\le\tfrac12, \\
	  &B_1(\nu)=
	  \tfrac1{8\nu} &&\text{\ \;if\ \;}\tfrac12\le\nu\le1, \\
	&\frac{(17 - 2 \nu) \left( 57187-3253 \nu-138 \nu ^2 + 212 \nu ^3 - 8 \nu ^4\right)}{6480000} &&\text{\ \;if\ \;}1\le\nu\le3.  
	\end{alignedat}
	\right.   
\end{multline} 
\end{proposition}

\begin{remark}\label{rem:B_2,B_3}
In particular, $\hat B_2(1)=B_1(1)=0.125$, $\hat B_2(3)=0.087018\ldots=
\frac{0.15072\dots}{\sqrt3}$, and $B_1(3)=0.087019\ldots=\frac{0.15072\dots}{\sqrt3}$ (cf.\ \eqref{eq:IS>B_2}). More generally, the choices $b=1$ for $\nu\in(0,\frac12]$ and $b=\tfrac1{2\nu}$ for $\nu\in[\frac12,1]$ 
in the proof of Proposition~\ref{prop:B_2,B_3} 
are optimal, in the sense that $\hat B_2(\nu)=B_1(\nu)$ for $\nu\in(0,1]$, as indicated in \eqref{eq:IS>B_2,nu<3}. The choice $b=\tfrac1{30}\, (17 - 2 \nu)$ for $\nu\in[1,3]$ in the just mentioned proof is nearly optimal; namely, then 
$\hat B_2(\nu)>B_1(\nu)-2\times10^{-6}$, for all $\nu\in[1,3]$; see details on this remark in Section~\ref{proofs}, right after the proof of Proposition~\ref{prop:B_2,B_3}.  
Of course, one can also rather easily give an exact algebraic expression for $B_1(\nu)$ with $\nu\in[1,3]$; however, that expression (in terms of certain roots of certain polynomials in one variable whose coefficients are polynomials in $\nu$) is complicated and therefore omitted here. 
\end{remark}

Theorem~\ref{th:B_2,tB_2}, Remark~\ref{rem:cheb,aver}, relations \eqref{eq:c_nu,tc_nu}, Proposition~\ref{prop:B_2,B_3}, and 
Re- \break mark~\ref{rem:B_2,B_3} are illustrated in Fig.\ \ref{fig:}. 

\begin{figure}[htbp]
	\centering
		\includegraphics[width=1.00\textwidth]{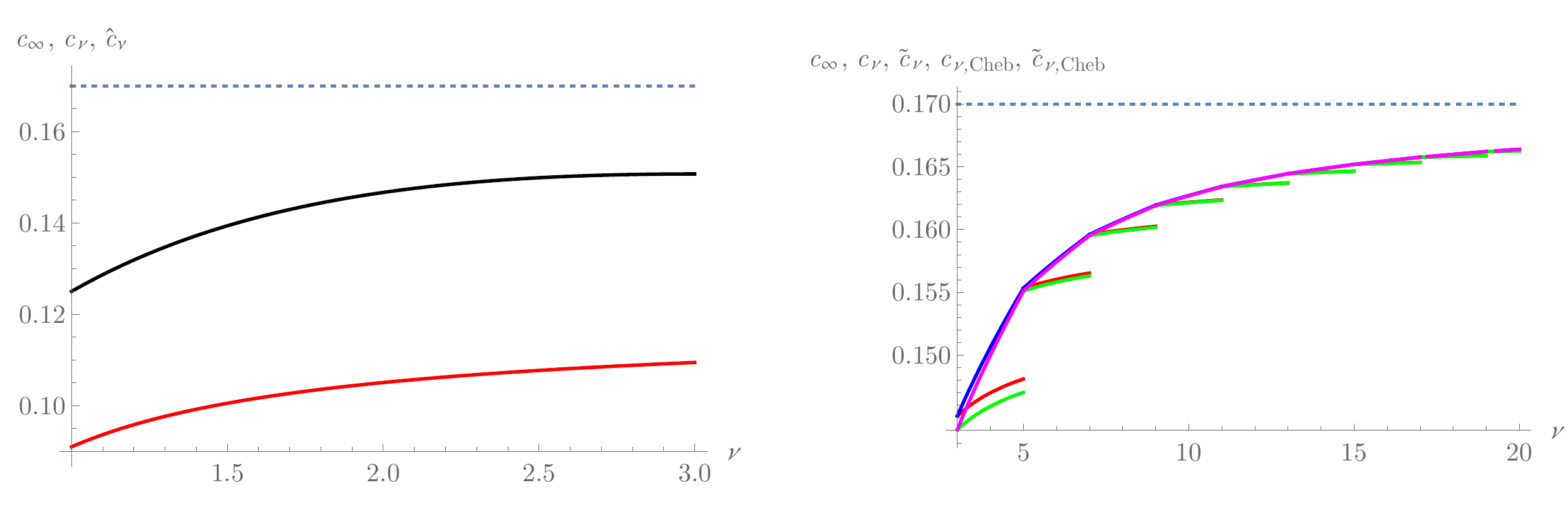}
	\caption{
	Left panel: graphs of $c_\nu$ (red) and $\hat c_\nu:=\sqrt\nu\,\hat B_2(\nu)$ (black) for $\nu\in[1,3]$. 
	Right panel: graphs of $c_\nu$ (red), $\tc_\nu$ (green), \emph{$c_{\nu,\text{Cheb}}$} (blue), and \emph{$\tc_{\nu,\text{Cheb}}$} (magenta) for $\nu\in[3,20]$, where \emph{$c_{\nu,\text{Cheb}}$} and \emph{$\tc_{\nu,\text{Cheb}}$} are obtained from the expressions for $c_\nu$ and $\tc_\nu$ in \eqref{eq:c_nu} and \eqref{eq:tc_nu} by replacing there $C_{i_\nu}$ and $\tC_{i_\nu}=\big(\frac{i_\nu +1}{i_\nu}\big)^{1/8}$ by the ``Chebyshev'' expressions $(1-w_i)C_i+w_i C_{i+1}$ and $(1-w_i)\tC_i+w_i \tC_{i+1}$, with $i=i_\nu$ and $w_i:=\tfrac{\nu-2i-1}2$, as discussed in Remark~\ref{rem:cheb,aver}.  
	The dotted horizontal line in both panels is at the level of $c_\infty=0.16997\dots$.  
	}
	\label{fig:}
\end{figure}

Let us also present the following very simple, but suboptimal, lower bound --- cf.\ e.g.\ \eqref{
eq:IS sim}. 

\begin{proposition}\label{prop:0.125}
If 
$\nu\ge
\frac3{41}$, then 
\begin{equation}\label{eq:0.125}
\inf_{L} \sup_{D}
\RR(L,D)
\ge B_0(m,d)\ge\frac{0.125}{\sqrt\nu}.       
\end{equation} 
\end{proposition} 

Note that the restriction $\nu\ge\frac3{41}$ in Proposition~\ref{prop:0.125} cannot be dropped, and it is in fact rather close to necessity. Indeed, in view of Remark~\ref{rem:upper}, the lower bound $\frac{0.125}{\sqrt\nu}$ in \eqref{eq:0.125} cannot hold for $\nu<\frac1{16}=\frac3{48}$. 
The lower bound $\frac{0.125}{\sqrt\nu}$ in \eqref{eq:0.125} was obtained by different methods in \citet{audibert2009} (under the condition $\nu\ge\frac14$, which was omitted from the paper but stated in a corrigendum). 

In conclusion of this section,
we summarize
the asymptotic behavior of the lower bounds 
$B_0(m,d), B_1(\nu),B_2(\nu),\tB_2(\nu)$ on the minimax 
EER, as well as the asymptotic behavior of the minimax 
EER itself.  

\begin{theorem}\label{th:asymp} 
\begin{equation}\label{eq:B_0 asymp}
\frac{c_\infty}{\sqrt\nu}\sim 
\inf_{L} \sup_{D}
\RR(L,D)
\ge B_0(m,d)\ge B_1(\nu)\ge B_2(\nu)\ge\tB_2(\nu)\sim\frac{c_\infty}{\sqrt\nu}      
\end{equation}
as $m$ and $d$ vary in any way such that 
$\nu=m/d\to\infty$.  
\end{theorem}

Thus, in view of \eqref{eq:clb}, the limit relation in \eqref{eq:clb->c_infty} holds and, moreover, all the lower bounds 
$B_0(m,d)$, $B_1(\nu)$, $B_2(\nu)$, $\tB_2(\nu)$ on the minimax 
EER are asymptotically equivalent to the minimax 
EER itself whenever $\nu=m/d\to\infty$. 
Clearly, Theorem~\ref{th:asymp} complements Theorem~\ref{th:negl}. 

\section{Proofs}\label{proofs} 

In this section, we shall prove (or provide details for) Theorems~\ref{th:duality} and \ref{th:average}, Proposition~\ref{lem:convex minor}, Theorem~\ref{th:B_2,tB_2}, 
Proposition~\ref{prop:B_2,B_3}, Remark~\ref{rem:B_2,B_3}, Proposition~\ref{prop:0.125},  Theorems~\ref{th:negl} and \ref{th:asymp} (together), and finally Theorem~\ref{th:P low}, in this order.  

\begin{proof}[Proof of Theorem~\ref{th:duality}]
The first equality in \eqref{eq:IS<} can be obtained using the von Neumann minimax duality theorem for bilinear functions on the product of simplexes \cite{von-Neumann} (plus a certain symmetrization argument); more general minimax duality theorems, for convex-concave-like functions, were given in \cite{sion}, and in \cite{pin-games-transl} a necessary and sufficient condition for the minimax duality for such functions was given. 

However, here we are going to offer a more direct and explicit argument, using the explicit form of the to-be-proved
-minimax decision rule $L^*_\erm$, as defined in \eqref{eq:L*_erm}. 

To gain some insight, let us begin with the simple case $d=1$. In that case, $\X=[d]=[1]=\{1\}$, $p=(1)$ (that is, $p_1=1$) and $\g=(b)$ with $b:=\g_1\in[-1,1]$; also, in the just mentioned definition \eqref{eq:L*_erm} of $L^*_\erm$, the terms $x$, $v_x$, $i_x$, and $n_x$ simplify, respectively, to $1$, $v_1=\sum_{i=1}^m y_i$, $1$, and $n_1=m$, in accordance with the definitions of $v_x$, $i_x$, and $n_x$ in \eqref{eq:v_x} and \eqref{eq:n_x,i_x}. 
Here we also have $X_i=1$ for all $i$ and hence, in view of \eqref{eq:Z_m^p,g} 
and \eqref{eq:Pbias}, $Z_m^{p,\g}
$ equals 
\begin{equation}\label{eq:Z_m^b}
	Z_m^{b}:=\big((1,Y_1^{b}),\dots,(1,Y_m^{b})\big)
\end{equation}
in distribution, where 
the $Y_i^{b}$'s are 
as in the statement of Theorem~\ref{th:duality}. 

A standard
argument 
(see e.g.\ \cite[\S1.8]{ferguson62_book})  
shows that $L^*_{1,\erm}$ 
is an optimal 
Bayesian 
decision rule, in the sense of being 
a minimizer 
of the 
average \break 
$\tfrac12\,\tsuml_{y\in\Y}\P\big(L(Z_m^{|b|y},U)(1)\ne y\big)$ 
of the types I and II error probabilities over all $L\in\Lrandm1$, that is, over all randomized learning algorithms $L$ for $d=1$. 
By symmetry, without loss of generality $b\ge0$, and so, $b\in[0,1]$. 
Then for the corresponding Bayes risk, 
given by the expression $\tfrac12 \,\tsuml_{y\in\Y}\P\big(L^*_{1,\erm}(Z_m^{by},U)(1)\ne y\big)$, 
for $m\ge1$ one has 
\begin{equation}
	\label{eq:def-opt}
\begin{aligned}	
  2
  \cdot
  \tfrac12 &\,\tsuml_{y\in\Y}\P\big(L(Z_m^{|b|y},U)(1)\ne y\big) \\ 
  \ge2
\cdot
  \tfrac12 &\,\tsuml_{y\in\Y}\P\big(L^*_{1,\erm}(Z_m^{by},U)(1)\ne y\big) \\ 
=&\P(V^b_m<0)+\P(V^b_m=0,Y^b_1<0)   
+\P(V^{-b}_m>0)+\P(V^{-b}_m=0,Y^{-b}_1>0) \\  
=&\P(V^b_m<0)+\P(V^b_m=0) 
+\P(V^{-b}_m>0) 
=\P(V^b_m\le0)
+\P(V^{-b}_m>0) \\ 
=&1-\big(\P(V^b_m>0)
-\P(V^{-b}_m>0)\big)
=2\opt(m,b),   
\end{aligned} 
\end{equation}
in accordance with 
\eqref{eq:V^b_k} \big(implying, in particular, that $(Y^{-b}_1,V^{-b}_m)$ equals \break 
$(-Y^b_1,-V^b_m)$ in distribution\big), 
\eqref{eq:opt:=}, \eqref{eq:s_k:=}, 
and the assumption $b\in[0,1]$ 
\big(which implies $\P(V^b_m>0)\ge\P(V^{-b}_m>0)$, since $Y_i^{b}$ is stochastically increasing in $b$\big). 
Thus, the Bayes risk $\tfrac12 \,\tsuml_{y\in\Y}\P\big(L^*_{1,\erm}(Z_m^{by},U)(1)\ne y\big)$ equals $\opt(m,b)$ for $b\in[0,1]$  
and $m\ge1$. This conclusion also trivially holds for $m=0$ 
(in which case $\opt(m,b)=\frac12$). 

Moreover, for 
each $m\in\intr0\infty$, the Bayes rule $L^*_{1,\erm}$ is a risk equalizer, in the sense that 
\begin{equation}\label{eq:equaliz}
	\text{for $b\in[0,1]$ and $y\in\{-1,1\}$,}\ \P\big(L^*_{1,\erm}(Z_m^{by},U)(1)\ne y\big)=\opt(m,b),  
\end{equation}
which does not depend on the choice of $y$; this conclusion follows because  
(i) \break
$(L^*_{1,\erm}\big((Z_m^{b})^-,-U\big)=-L^*_{1,\erm}(Z_m^{b},U)$, where $(Z_m^{b})^-:=
\big((1,-Y_1^{b}),\dots,\break 
(1,-Y_m^{b})\big)$,  
and (ii) the distribution of 
$(Y_1^{-b},\dots,Y_k^{-b},
-U)$ is the same as that of $-(Y_1^{b},\dots,Y_k^{b},U)$.

Let us now proceed to the general case of any natural $d$, which
in a sense reduces
to the case $d=1$. 
Take any $m\in\intr0\infty$, any randomized learning algorithm $L\colon(\X\times\Y)^m\times[-1,1]\to\H$, any $p\in[0,1]^{[d]}$ such that $\sum_{x=1}^d p_x=1$, and any $\g\in[-1,1]^{[d]}$. 
For each $x\in[d]$, introduce the random set 
\begin{equation}\label{eq:JJ}
	\J^p_x:=\{i\in[m]\colon X_i^p=x\}
\end{equation}
and its cardinality 
\begin{equation}\label{eq:N_x}
	N^p_x:=\card\J^p_x. 
\end{equation}
%
Then, by 
\eqref{eq:E De_emp},
\begin{align}
\RR_m(L;p,\g)
	&=\sum_{x=1}^d p_x|\g_x| \sum_{k=0}^m\P\big(L\big(Z_m^{p,\g},U\big)(x)\ne h_\g(x),N^p_x=k\big). \label{eq:E De}  
\end{align}
Next, take any $x\in[d]$ and any $k=0,\dots,m$. Then  
\begin{multline}\label{eq:P(..,N_x)} 
	\P\big(L\big(Z_m^{p,\g},U\big)(x)\ne h_\g(x),N^p_x=k\big) \\ 
	=\sum_{J\in\binom{[m]}k}\P\big(L\big(Z_m^{p,\g},U\big)(x)\ne h_\g(x),\J^p_x=J\big), 
\end{multline} 
where $\binom{[m]}k:=\{J\subseteq[m]\colon\card J=k\}$. 

Further, take any set $J\in\binom{[m]}k$. 
Writing $J$ as $\{i_1,\dots,i_k\}$ with $i_1<\dots<i_k$, let $X_J^p:=(X_{i_1}^p,\dots,X_{i_k}^p)$, and similarly define $X_{J^\cc}^p$, $Y_J^{p,\g}$, and $Y_{J^\cc}^{p,\g}$, where $J^\cc:=[m]\setminus J$. 
Let also $Z_{J^\cc}^{p,\g}:=(X_{J^\cc}^p,Y_{J^\cc}^{p,\g})$. 
For any $x\in\X$, let $x^J:=(x,\dots,x)\in\X^k$ and $\Z_x:=(\X\setminus\{x\})\times\Y$.  
Then, in view of \eqref{eq:h_ga}, 
\begin{multline}\label{eq:P(..,J)}
	\P\big(L\big(Z_m^{p,\g},U\big)(x)\ne h_\g(x),\J^p_x=J\big) \\ 
	=\!\!\sum_{z\in\Z_x^{m-k}}
	\P\big(L\big(Z_m^{p,\g},U\big)(x)\ne \sign\g_x|X_J^p=x^J,Z_{J^\cc}^{p,\g}=z\big) 
	\P\big(X_J^p=x^J,Z_{J^\cc}^{p,\g}=z\big). 
\end{multline}
\big(Here and in what follows, to simplify the writing, we neglect the possibility that $\P(X_J^p=x^J,Z_{J^\cc}^{p,\g}=z)$ may equal $0$. Of course, in such cases we may let the corresponding conditional probabilities in \eqref{eq:P(..,J)} take whatever values 
deemed most suitable for us at any given point.\big)

By \eqref{eq:E De}, \eqref{eq:P(..,N_x)}, and \eqref{eq:P(..,J)}, 
\begin{multline}\label{eq:RR=}
	\RR_m(L;p,\g)  
=\sum_{x,k,J,z}p_x|\g_x|\,
	\P\big(L\big(Z_m^{p,\g},U\big)(x)\ne \sign\g_x|X_J^p=x^J,Z_{J^\cc}^{p,\g}=z\big) \\ 
	\times\P\big(X_J^p=x^J,Z_{J^\cc}^{p,\g}=z\big),    
\end{multline}
where $\sum\limits_{x,k,J,z}:=\sum\limits_{x=1}^d\, \sum\limits_{k=0}^m\, \sum\limits_{J\in\binom{[m]}k}\,   
\sum\limits_{z\in\Z_x^{m-k}}$,  
and, moreover, the sum 
\begin{equation}\label{eq:=P(N=k)}
	\sum_{J,z}\P\big(X_J^p=x^J,Z_{J^\cc}^{p,\g}=z\big)=\P(N^p_x=k)  
\end{equation}
\big(where $\sum\limits_{J,z}:=\sum\limits_{J\in\binom{[m]}k} \,   
\sum\limits_{z\in\Z_x^{m-k}}$\big) does not depend on $\g$. To quickly see why identity \eqref{eq:=P(N=k)} holds, look back at \eqref{eq:P(..,N_x)} and \eqref{eq:P(..,J)}, with the event 
$\{L\big(Z_m^{p,\g},U\big)(x)\ne h_\g(x)\}$ replaced there by an event of probability $1$. 

Since $(X_1^p,Y_1^{p,\g}),\dots,(X_m^p,Y_m^{p,\g}),U$ are independent, 
for any $L\in\Lrand$, $x\in[d]$, $k\in\intr0m$, $J\in\binom{[m]}k$, 
$z\in\Z_x^{m-k}$, and $p\in[0,1]^{[d]}$ such that $\sum_{x=1}^d p_x=1$, the conditional probability 
$\P\big(L\big(Z_m^{p,\g},U\big)(x)\ne \sign\g_x|X_J^p=x^J,Z_{J^\cc}^{p,\g}=z\big)$ in \eqref{eq:P(..,J)} depends on $\g$ only through $\g_x$, whereas the unconditional probability 
$\P\big(X_J^p=x^J,Z_{J^\cc}^{p,\g}=z\big)$ in \eqref{eq:P(..,J)} depends on $\g$ only through 
$\g_{\setminus x}:=\g\big|_{\X\setminus\{x\}}$
--- the restriction of the function $\g$ to the subset $\X\setminus\{x\}$ of the set $\X$. 
So, introducing the averaging operators  
\begin{equation*}
	\A_\si:=\frac1{2^d}\sum\limits_{\si\in\{-1,1\}^{[d]}}, \quad 
	\A_{\si_{\setminus x}}:=\frac1{2^{d-1}}\sum\limits_{\si_{\setminus x}\in\{-1,1\}^{\X\setminus\{x\}} }, \quad 
	\A_{\si_x}:=\frac12\sum\limits_{\si_x\in\{-1,1\}}, \quad 
\end{equation*}
in view of 
\eqref{eq:RR=} one has  
\begin{multline}\label{eq:ave=}
	\A_\si\RR_m(L;p,|\g|\si)  \\ 
=\sum_{x,k,J,
z}p_x|\g_x|
\A_{\si_x}\P\big(L\big(Z_m^{p,
|\g|\si},U\big)(x)\ne \si_x|X_J^p=x^J,Z_{J^\cc}^{p,|\g|\si}=z\big) \\
\times\A_{\si_{\setminus x}}\P\big(X_J^p=x^J,Z_{J^\cc}^{p,
|\g|\si}=z\big). 
\end{multline}  

Recall that the random pairs $(X_1^p,Y_1^{p,\g}),\dots,(X_m^p,Y_m^{p,\g})$ are independent copies of the random pair $(X^p,Y^{p,\g})=(X^D,Y^D)$ satisfying condition \eqref{eq:Pbias}, and the r.v.\ $U$ is independent of these pairs. So, for any 
$x\in[d]$, $k\in\intr0m$, $J\in\binom{[m]}k$, 
$z\in\Z_x^{m-k}$, 
and $p\in[0,1]^{[d]}$ such that $\sum_{x=1}^d p_x=1$, 
\label{cond}
the conditional distribution of $(Y_J^{p,\g},
U)$ given $X_J^p=x^J$ 
and $Z_{J^\cc}^{p,\g}=z$ is the same as 
the distribution of $(Y_1^{\g_x},\dots,Y_k^{\g_x},U)$, where 
the $Y_i^{b}$'s are again as in the statement of Theorem~\ref{th:duality}.  
Therefore, in view of \eqref{eq:def-opt}, 
for each $x\in[d]$,  
each $z=\big((x_{k+1},y_{k+1}),\dots,(x_m,y_m)\big)\in\Z_x^{m-k}$, and $J=[k]$, 
\begin{multline}\label{eq:L,cond}
\A_{\si_x}\P\big(L\big(Z_
k^{p,|\g|\si},U\big)(x)\ne \si_x|X_J^p=x^J,Z_{J^\cc}^{p,|\g|\si}=z\big) \\
=\A_{\si_x}\P\big(L_x\big(Z_
k^{|\g_x|\si_x},U\big)(1)\ne \si_x\big) 
\ge\opt(k,|\g_x|),  
\end{multline}
where $Z_
k^{b}$ is as 
in \eqref{eq:Z_m^b} and 
$L_x\colon([1]\times\Y)^
k\times[-1,1]\to\{-1,1\}^{[1]}$ is the random\-ized learning algorithm (for the case $d=1$) defined by the formula \break 
$L_x\big(((1,y_1),\dots,(1,y_
k)),u\big):=h|_{\{x\}}$, the restriction of the function $h$ to the singleton set $\{x\}$, where 
$h:=h_{
w;u}:=L\big(
w,u\big)\in\{-1,1\}^{[d]}$, 
$w:=\big((x,y_1),\dots,(x,y_k),\break 
(x_{k+1},y_{k+1}),\dots,(x_m,y_m)\big)$, and $u\in[-1,1]$. 
Clearly then, \eqref{eq:L,cond} holds for any $J\in\binom{[m]}k$. 

Similarly, 
but now using \eqref{eq:equaliz} instead of \eqref{eq:def-opt}, 
we have 
\begin{equation}\label{eq:L^*,cond}
\P\big(L^*_{m,\erm}\big(Z_m^{p,|\g|\si},U\big)(x)\ne \si_x|X_J^p=x^J,Z_{J^\cc}^{p,|\g|\si}=z\big)
=\opt(k,|\g_x|)   
\end{equation}
for each $x\in[d]$ \emph{and each} $\si_x\in\{-1,1\}$, and hence the average of the left-hand side (l.h.s.) of \eqref{eq:L^*,cond} over $\si_x\in\{-1,1\}$ equals $\opt(k,|\g_x|)$ as well. 
%
%
Note also that $
\P\big(X_J^p=x^J,Z_{J^\cc}^{p,|\g|\si}=z\big)$ does not depend on $L$. So, collecting identity \eqref{eq:ave=}, its counterpart with $L^*_{m,\erm}$ in place of $L$, \eqref{eq:L,cond}, and identity \eqref{eq:L^*,cond} with its l.h.s.\ replaced by the average of that l.h.s.\ over $\si_x\in\{-1,1\}$, we see 
that 
\begin{equation}\label{eq:ave>ave}
	\A_\si\RR_m(L;p,|\g|\si)
	\ge \A_\si\RR_m(L^*_{m,\erm};p,|\g|\si). 
\end{equation}
Moreover, 
by \eqref{eq:RR=} with $L^*_{m,\erm}$ in place of $L$, \eqref{eq:L^*,cond} with 
$\si=\sign\g$, and  \eqref{eq:=P(N=k)}, 
\begin{equation}\label{eq:RR*=}
\begin{aligned}
\RR_m(L^*_{m,\erm};p,\g)  
=&\sum_{x,k,J,z}p_x|\g_x| \opt(k,|\g_x|)\,
\P\big(X_J^p=x^J,Z_{J^\cc}^{p,\g}=z\big)
 \\ 
	=&\sum_{x=1}^d p_x|\g_x| \sum_{k=0}^m\opt(k,|\g_x|) \sum_{J,z}
	\P\big(X_J^p=x^J,Z_{J^\cc}^{p,\g}=z\big)
	\\	
		=&\sum_{x=1}^d p_x|\g_x| \sum_{k=0}^m\opt(k,|\g_x|) \P(N^p_x=k) 
	\\	   
=&\sum_{x=1}^d p_x|\g_x| \E\opt(N^p_x,|\g_x|), 
\end{aligned}   
\end{equation}
which proves \eqref{eq:RR*} 
and the second equality in \eqref{eq:IS<} (here one may recall   
the definition of $\RR_m(L;p,\g)$ in \eqref{eq:E De_emp}).  
We also see that $\RR_m(L^*_{m,\erm};p,\g)$ depends on $\g$ only through $|\g|$. 
This and \eqref{eq:ave>ave} yield  
\begin{multline*}
\max_{\si\in\{-1,1\}^d}\RR_m(L^*_{m,\erm};p,|\g|\si)
	=\A_\si\RR_m(L^*_{m,\erm};p,|\g|\si) \\ 
=\sum_{x=1}^d p_x|\g_x| \E\opt(N^p_x,|\g_x|)	\le\A_\si\RR_m(L;p,|\g|\si)  
	\le\max_{\si\in\{-1,1\}^d}\RR_m(L;p,|\g|\si). 
\end{multline*}
Taking now $\sup_{p,\g}$, 
we see that 
$\sup_D\RR_m(L^*_{m,\erm},D)\le\sup_D\RR_m(L,D)$ for all $L$, which 
proves the first equality in \eqref{eq:IS<}. This 
completes the proof 
of Theorem~\ref{th:duality}. 
\end{proof} 

\begin{proof}[Proof of Theorem~\ref{th:average}] 
This theorem follows immediately from \eqref{
eq:ave>ave} and \eqref{eq:RR*=} by 
taking there $p_x=\frac1d$ for all $x\in[d]$ and any $\g\in\{-b,b\}^{[d]}$. 
\end{proof}

\begin{proof}[Proof of Proposition~\ref{lem:convex minor}] 
By \eqref{eq:opt:=}--\eqref{eq:s_k:=}
, $\opt(k,b)=1/2$ for all $k=0,1,\dots$ if $b=0$. 
Suppose now that $b\in(0,1]$. 
\label{LLN}
Letting $k\to\infty$ and using the law of large numbers, in view of \eqref{eq:V^b_k} we have $\frac1k\,V^b_k\to\E Y^b_1=b>0$ in probability. 
So, $\P(V_k^b>0)\to1$; 
similarly, $\P(V_k^{-b}>0)\to0$. 
Recalling \eqref{eq:opt:=}--\eqref{eq:s_k:=} again, we see that 
$\opt(k,b)\to0$. 
So, by Lemma~\ref{lem:concave} in Appendix~\ref{app:bayes}, $\opt(k,b)$ is convex and 
nonincreasing in $k\in\{0,1,3,5,\dots\}$, for each $b\in[0,1]$. 
It remains to use relations \eqref{eq:s_0,s_1,s_3} and \eqref{eq:odd-even} in Appendix~\ref{app:bayes} and, again, \eqref{
eq:opt:=}.  
\end{proof}

\begin{proof}[Proof of Theorem~\ref{th:B_2,tB_2}]
The first inequality in \eqref{eq:IS>B_2} comes from \eqref{eq:IS>hatopt}. The second inequality in \eqref{eq:IS>B_2} follows immediately from Lemma~\ref{
lem:hatopt} with $\ka=\nu$ and $b=z_*/\sqrt\nu$. 
The inequality in \eqref{eq:c_nu} holds because, by Lemma~\ref{lem:C_i}, $C_i>1$ and, by 
the paragraph containing \eqref{eq:z*}, $c_\infty=\tfrac{z_*}2\,\big(1-\erf(z_*/\sqrt2)\big)$. 
The last equality in \eqref{eq:c_nu} follows by \eqref{eq:C_i}. 
The inequalities in \eqref{eq:tc_nu} follow immediately from \eqref{eq:IS>B_2}, \eqref{eq:c_nu}, and \eqref{eq:C_i<}.  
\end{proof}

\begin{proof}[Proof of Proposition~\ref{prop:B_2,B_3}]
By \eqref{eq:hatopt:=}, \eqref{
eq:opt:=}, and \eqref{eq:s_0,s_1,s_3}, 
\begin{equation*}
	\hatopt(\nu)=
		\left\{
	\begin{alignedat}{2}
	&
	\tfrac12\,(1-\nu b) &&\text{\ \;if\ \;}0<\nu\le1, \\  
	&\tfrac18\, \big(4 + b^3 (\nu-1 ) - b (3 + \nu)\big) &&\text{\ \;if\ \;}1\le\nu\le3.  
	\end{alignedat}
	\right.
\end{equation*}
Recalling now \eqref{eq:IS>hatopt} and using the values 
$b=1$ for $\nu\in(0,\frac12]$, $b=\frac1{2\nu}$ for $\nu\in[\frac12,1]$, and 
and $b=\tfrac1{30}\, (17 - 2 \nu)$ for $\nu\in[1,3]$, one obtains \eqref{eq:IS>B_2,nu<3}.  
\end{proof}

\begin{proof}[Details on Remark~\ref{rem:B_2,B_3}]
The inequality  
$\hat B_2(\nu)>B_1(\nu)-2\times10^{-6}$ for all $\nu\in[1,3]$, mentioned in that remark, can be verified, e.g., by issuing the Mathematica command 
\verb!Reduce[b 1/8 (4 + b^3 (nu - 1) - b (3 + nu)) - hB2[nu]! \break
\verb!>= 2 10^(-6) && 1 <= nu <= 3 && 0 <= b <= 1]!, 
where \verb!hB2[nu]! 
stands for $\hat B_2(\nu)$. 
This command then outputs 
\verb!False!, which means that indeed $B_1(\nu)=\max\limits_{0\le b\le1}b\,\hatopt(\nu)=\max\limits_{0\le b\le1} b\,\tfrac18\, (4 + b^3 (\nu-1 ) - b (3 + \nu))<\hat B_2(\nu) 
+2\times10^{-6}$. 
\end{proof}

\begin{proof}[Proof of Proposition~\ref{prop:0.125}]
The first inequality in \eqref{eq:0.125} holds by \eqref{eq:IS>}. Take now any $b\in(0,1]$. 
From \eqref{eq:s_k=}, Lemma~\ref{lem:C_i}, and inequality $S_q(b)\le b$, it follows that 
\begin{equation}\label{eq:<2b sqrt k}
	s_k(b)\le b\sqrt k
\end{equation}
for odd natural $k$. By \eqref{eq:odd-even}, inequality \eqref{eq:<2b sqrt k} holds for even natural $k$ as well, and it trivially holds for $k=0$. 
Using now the definition of $B_0(m,d)$ in \eqref{eq:IS>} together with \eqref{
eq:opt:=}, \eqref{eq:<2b sqrt k}, and Jensen's inequality, 
noticing that $\frac1{2\sqrt\nu}\in(0,1]$ if $\nu\ge\frac14$, and substituting $\frac1{2\sqrt\nu}$ for $b$, one has 
\begin{multline*}
	B_0(m,d)\ge b\,\E\opt(N,b)\ge\tfrac b2\,\big(1-b\E\sqrt N\big)
	\ge\tfrac b2\,\big(1-b\sqrt{\E N}\big) \\ 
	=\tfrac b2\,\big(1-b\sqrt{\nu}\big)
	=\frac{0.125}{\sqrt\nu},  
\end{multline*}
in the case when $\nu\ge\frac14$. 

It remains to consider the case when $\frac14>\nu\ge\frac3{41}$. Then, by \eqref{eq:IS>hatopt} and \eqref{eq:IS>B_2,nu<3}, 
\begin{equation*}
	B_0(m,d)\ge B_1(\nu)=\tfrac12\,(1-\nu)>\frac{0.125}{\sqrt\nu},   
\end{equation*}
which completes the proof of Proposition~\ref{prop:0.125}. 
\end{proof}

\begin{proof}[Proof of Theorems~\ref{th:negl} and \ref{th:asymp}]
The two inequalities in \eqref{eq:IS sim} are trivial. 
The first, second, third, and fourth inequalities in \eqref{eq:B_0 asymp} were already established as the inequalities in \eqref{eq:IS>}, the second inequality in \eqref{eq:IS>hatopt}, the second inequality in \eqref{eq:IS>B_2}, and the first inequality in \eqref{eq:tc_nu}, respectively. 
The second asymptotic equivalence in \eqref{eq:B_0 asymp} follows immediately from 
\eqref{eq:tc_nu} and \eqref{eq:c_nu,tc_nu}. 

So, in view of \eqref{eq:IS<}, it suffices to show that \eqref{eq:ties} holds and 
\begin{equation}\label{eq:B sim}
	B(m,d)\overset{\text{?}}\lesssim\frac{c_\infty}{\sqrt{\nu}}, 
\end{equation}
where, as usual, $A\lesssim B$ means $\limsup A/B\le1$. 

{
\setlength{\fboxsep}{5pt}
\noindent
\framebox{\parbox{.95\textwidth}{
In this proof, 
\label{framed}
all the limit relations are stated for $\nu=
m/d\to\infty$ and all \break 
the other relations are stated under the condition that $\nu$ is large enough. 
}
}
}

By \eqref{eq:E De_emp}, \eqref{eq:L_erm}, 
\eqref{eq:V_x}, and \eqref{eq:h_ga},     
\begin{equation}\label{eq:Eh_M}
\begin{aligned}
	\RR(L_\erm;p,\g)
	&\le\sum_{x=1}^d p_x|\g_x|
		\big[\P(V^{p,\g}_x\ne0,\sign V^{p,\g}_x\ne\sign\g_x)
	+\P(V^{p,\g}_x=0)\big]. 
\end{aligned}	 
\end{equation}
Take now any $b\in[
0,1]$ and $k\in\intr0\infty$. 
In view of \eqref{eq:V^b_k}, $V^b_k$ equals $-V^{-b}_k$ in distribution, and $V^b_k$ is stochastically greater than $V^{-b}_k$ (since $b\ge0$). 
So,  
by the definition 
\eqref{eq:opt:=}--\eqref{eq:s_k:=} of $\opt(k,b)$,  
\begin{equation}\label{eq:opt,Q__}
\opt(k,b)=Q_-(k,b)+\tfrac12\,Q_0(k,b),  
\end{equation}
where 
\begin{equation}\label{eq:Q_-,Q_0}
	  Q_-(k,b):=\P\big(V^b_k<0\big)\quad\text{and}\quad Q_0(k,b):=\P(V^b_k=0).     
\end{equation}

Recalling the definition of the random pairs $(X_1^p,Y_1^{p,\g}),\dots,(X_m^p,Y_m^{p,\g})$ in the paragraph containing \eqref{eq:Pbias} and \eqref{eq:Z_m^p,g}, the definition of the random set $\J^p_x$ in \eqref{eq:JJ}, and the definition of the $Y_i^{b}$'s in  
Theorem~\ref{th:duality}, we see that, 
for any $J\in\binom{[m]}k$,  
the conditional distribution of $(Y^{p,\g}_i)_{i\in J}$ 
given the event $\{\J^p_x=J\}\big[=\{X_i=x\ \forall i\in J\}\cup\{X_i\ne x\ \forall i\in J^\cc\}\big]$ is the same as the distribution of $(Y_1^{\g_x},\dots,Y_k^{\g_x})$. 

Therefore, 
if $\g_x>0$ for some $x\in[d]$, then, in view of \eqref{eq:N_x}, 
\eqref{eq:V_x}, and \eqref{eq:opt,Q__}, 
\label{comment14,R1R1}
\begin{align*}
&\P(V^{p,\g}_x\ne0,\sign V^{p,\g}_x\ne\sign\g_x)
=\P(V^{p,\g}_x<0) \\ 
&=\sum_{k=0}^m\P(V^{p,\g}_x<0,N^p_x=k) \\
&=\sum_{k=0}^m \sum_{J\in\binom{[m]}k}\P\big(V^{p,\g}_x<0,\J^p_x=J\big) \\  
&=\sum_{k=0}^m \sum_{J\in\binom{[m]}k}\P\Big(\sum_{i\in J} Y^{p,\g}_i<0,\J^p_x=J\Big) \\  
&=\sum_{k=0}^m \sum_{J\in\binom{[m]}k}\P\Big(\sum_{i=1}^k Y^{\g_x}_i<0\Big)\,\P(\J^p_x=J) \\  
&=\sum_{k=0}^m  \P\Big(\sum_{i=1}^k Y^{\g_x}_i<0\Big)\;\P(N^p_x=k) 
=\sum_{k=0}^m Q_-(k,|\g_x|)\P(N^p_x=k) 
\\ 
&=\E Q_-(N^p_x,|\g_x|)=\E\opt(N^p_x,|\g_x|)-\tfrac12\,\E Q_0(N^p_x,|\g_x|).    	
\end{align*}
Similarly, the latter expression, 
$\E\opt(N^p_x,|\g_x|)-\tfrac12\,\E Q_0(N^p_x,|\g_x|)$, for \break 
$\P(V^{p,\g}_x\ne0,\sign V^{p,\g}_x\ne\sign\g_x)$ holds when $\g_x<0$ as well. On the other hand, it is similarly seen that 
\begin{equation*}
	\E Q_0(N^p_x,|\g_x|)=\P(V^{p,\g}_x=0). 
\end{equation*}
Thus, \eqref{eq:Eh_M} yields   
\begin{equation*}
\begin{aligned}
\RR(L_\erm;p,\g)	&\le\sum_{x=1}^d p_x|\g_x|\E\opt(N^p_x,|\g_x|)+\frac12\,\sum_{x=1}^d p_x|\g_x|\E Q_0(N^p_x,|\g_x|). 
\end{aligned}	 
\end{equation*} 
In particular, in view of \eqref{eq:IS<}, this implies the second inequality in \eqref{eq:ties}; the first inequality there is trivial. 

Now, to complete the proof of 
\eqref{eq:ties} and Theorem~\ref{th:asymp}, it remains to verify \eqref{eq:B sim} and 
\begin{equation}\label{eq:EQ_0}
	\sup_{p,\g}\sum_{x=1}^d p_x|\g_x|\E Q_0(N^p_x,|\g_x|)\overset{\text{?}}=O(1/\nu). 
\end{equation}

Take any $b\in(0,1]$ and any natural $k\ge3$, so that, by \eqref{eq:q}  
and \eqref{eq:j}, $q:=q_k\ge1>0$. 
Note that $s_k(1)=1$. Hence, by \eqref{
eq:opt:=} 
and \eqref{eq:s_k=}, 
\begin{multline}\label{eq:1-s(y)/2}
	\opt(k,b)=\tfrac12\big(1-s_k(b)\big)=\tfrac12\big(s_k(1)-s_k(b)\big)
	=\tfrac12\,s_k'(0)\big(S_q(1)-S_q(b)\big) \\ 
	=\tfrac12\,s_k'(0)\int_b^1(1-u^2)^q\,du  
	\le\tfrac12\,s_k'(0)\int_b^1 e^{-qu^2}\,du
	\le\tfrac12\,s_k'(0)\int_b^\infty e^{-qu^2}\,du 
	\\ 
	=A_k\big(1-\erf(b\sqrt{q}\,)\big),    
\end{multline}
where $A_k:=\frac{s_k'(0)\sqrt\pi}{4\,\sqrt q}\to\frac12$ as $k\to\infty$, by 
Lemma~\ref{lem:C_i} and \eqref{eq:s'odd-even}. 
Therefore, for $z:=b\sqrt{2q}$ one has 
\begin{equation}\label{eq:opt<}
	b\opt(k,b)\le\frac{\la_k}{\sqrt k}\,\frac z2\,\big(1-\erf(z/\sqrt{2}\,)\big)
	\le c_\infty\,\frac{\la_k}{\sqrt k} 
\end{equation}
by \eqref{
eq:clb->c_infty}, where 
\begin{equation}\label{eq:la_k to infty}
	\la_k:=2A_k\sqrt{\tfrac{k}{2q}}\to1
\end{equation}
as $k\to\infty$. 

Take now any $\vp\in(0,1)$. In view of \eqref{eq:opt<} and  
the first part of Remark~\ref{rem:upper}, 
\begin{equation}\label{eq:opt<<}
	b\opt(k,b)\le\frac A{\sqrt{k+1}} 
\end{equation}
for some real $A>0$, all $b\in[0,1]$, and all $k=0,1,\dots$.

Since the r.v.\ $N^p_x$ has the binomial distribution with parameters $m$ and $p_x$, one has 
the following well-known inequality: 
\begin{equation*}
	\P(N^p_x\le(1-\vp)mp_x)\le e^{-\vp^2 mp_x/2};   
\end{equation*}
see e.g.\ \cite[Exercise~2.9]{BLM}; 
also, this inequality immediately follows from the more general and precise results in \cite[(1.3) or (2.31)]{left_publ} or \cite[Theorem~7]{pin-utev-exp}. 

Therefore,
\begin{multline}\label{eq:S_01}
	S_{01}:=\sum_{x=1}^d p_x\E\frac1{\sqrt{N^p_x+1}}\ii{N^p_x\le(1-\vp)mp_x} \\ 
	\le\sum_{x=1}^d p_x\P(N^p_x\le(1-\vp)mp_x)
	\le\sum_{x=1}^d p_x e^{-\vp^2 mp_x/2}
		\le\sum_{x=1}^d {a_\vp}\frac1m
	={a_\vp}\frac dm,
\end{multline}
where $a_\vp
:=\frac2{\vp^2}\sup_{u>0}ue^{-u}=\frac2{e\vp^2}$, 
which depends only on $\vp$.
Also,  
\begin{multline}\label{eq:S_02}
	S_{02}:=\sum_{x=1}^d p_x\E\frac1{\sqrt{N^p_x+1}}\ii{N^p_x>(1-\vp)mp_x} \\ 
	\le\sum_{x=1}^d p_x\frac1{\sqrt{(1-\vp)mp_x+1}}
		\le\sum_{x=1}^d \frac{\sqrt{p_x}}{\sqrt{(1-\vp)m}}
		\le\frac1{\sqrt{1-\vp}}\,\sqrt{\frac dm}; 
\end{multline} 
the last inequality here is obtained using the concavity of the square root function together with the condition $\sum_{x=1}^d p_x=1$. 
Thus, by \eqref{eq:S_01} and \eqref{eq:S_02}, 
\begin{equation}\label{eq:S_0<}
	\sum_{x=1}^d p_x\E\frac1{\sqrt{N^p_x+1}}=S_{01}+S_{02}\le\frac1{1-\vp}\,\sqrt{\frac dm}
\end{equation}
(if $m/d$ is large enough, depending on $\vp$; recall the framed convention on page~\pageref{framed}). 

In view of \eqref{eq:opt<<}, 
\begin{multline*}
	S_{11}:=\sum_{x=1}^d p_x|\g_x|\E\opt(N^p_x,|\g_x|) \ii{N^p_x>(1-\vp)mp_x} \ii{p_x\le\vp/d} \\ 
\le\sum_{x=1}^d p_x \E\frac A{\sqrt{N^p_x+1}} \ii{N^p_x>(1-\vp)mp_x} \ii{p_x\le\vp/d} \\ 
%
	\le\sum_{x=1}^d p_x\frac A{\sqrt{(1-\vp)mp_x+1}} \ii{p_x\le\vp/d}
		\le A\sum_{x=1}^d \frac{\sqrt{p_x}}{\sqrt{(1-\vp)m}}\ii{p_x\le\vp/d} \\ 
	\le A\sum_{x=1}^d \frac{\sqrt{\vp/d}}{\sqrt{(1-\vp)m}}
=A\sqrt{\frac\vp{1-\vp}}\,\sqrt{\frac dm}.  
\end{multline*} 

Next, taking into account \eqref{eq:opt<}, \eqref{eq:la_k to infty},  and \eqref{eq:S_02}, one has 
\begin{multline*}
	S_{12}:=\sum_{x=1}^d p_x|\g_x|\E\opt(N^p_x,|\g_x|) \ii{N^p_x>(1-\vp)mp_x} \ii{p_x>\vp/d} \\ 
\le\sum_{x=1}^d p_x|\g_x|\E\opt(N^p_x,|\g_x|) \ii{N^p_x>(1-\vp)mp_x} \ii{N^p_x>(1-\vp)\vp m/d} \\ 
\le\sum_{x=1}^d p_x(1+\vp)c_\infty\E\frac1{\sqrt{N^p_x+1}} \ii{N^p_x>(1-\vp)mp_x} 
=(1+\vp)c_\infty S_{02} \\ 
\le\frac{(1+\vp)c_\infty}{\sqrt{1-\vp}}\,\sqrt{\frac dm}.   
\end{multline*}

So, 
\begin{multline}\label{eq:S_1}
S_1:=\sum_{x=1}^d p_x|\g_x|\E\opt(N^p_x,|\g_x|) \ii{N^p_x>(1-\vp)mp_x}
=S_{11}+S_{12} \\ 
\le \frac{(1+A_1\sqrt\vp)c_\infty}{\sqrt{1-\vp}}\,\sqrt{\frac dm}	
\end{multline}
for some universal real constant $A_1>0$. 

On the other hand, by \eqref{eq:opt<<} and \eqref{eq:S_01}, 
\begin{multline*}
S_2:=\sum_{x=1}^d p_x|\g_x|\E\opt(N^p_x,|\g_x|) \ii{N^p_x\le(1-\vp)mp_x} \\ 
\le AS_{01}\le A{a_\vp}\frac dm\le \vp\sqrt{\frac dm}.  
\end{multline*}

Combining this with \eqref{eq:S_1}, we conclude that 
\begin{equation}\label{eq:sum opt<}
	\sum_{x=1}^d p_x|\g_x|\E\opt(N^p_x,|\g_x|)
	=S_1+S_2\le \frac{(1+A_2\sqrt\vp)c_\infty}{\sqrt{1-\vp}}\,\sqrt{\frac dm} 
\end{equation}
for some universal real constant $A_2>0$. 
Thus, 
letting here $\vp$ be
arbitrarily
small and recalling the definition of $B(m,d)$ in \eqref{eq:IS<}, we complete the proof of the asymptotic relation \eqref{eq:B sim}. 

To complete the proof of Theorems~\ref{th:negl} and \ref{th:asymp}, let us finally verify \eqref{eq:EQ_0}. 
If $k=2j$ is even then, by \eqref{eq:Q_-,Q_0},  
\begin{equation*}
	b Q_0(k,b)=\binom{2j}j\frac1{4^j}\,b(1-b^2)^j\le \frac{A_3}{\sqrt{k+1}}\,b e^{-b^2k/2} 
	\le\frac{A_4}{k+1} 
\end{equation*}
for some universal real constants $A_3>0$ and $A_4>0$ and all $b\in[0,1]$; since $Q_0(k,b)=0$ if $k$ is odd, the above bound in fact holds for all $k=0,1,\dots$. 
So, 
\begin{equation*}
	\sum_{x=1}^d p_x|\g_x|\E Q_0(N^p_x,|\g_x|)
	\le A_4\sum_{x=1}^d p_x\E\frac1{N^p_x+1}
	\le A_4\Big(a_\vp+\frac1{1-\vp}\Big)\frac dm ; 
\end{equation*}
the second inequality in the above display is obtained similarly to the inequality in 
\eqref{eq:S_0<}. 
Thus, \eqref{eq:EQ_0} is verified, and the proof of Theorems~\ref{th:negl} and \ref{th:asymp} is complete. 
\end{proof}

\begin{proof}[Proof of Theorem~\ref{th:P low}]
Take any learning algorithm $L\in\Lrand$. 
Take any $b\in(0,1]$ and then any $\vp\in(0,b)$ 
and any $\g\in\{-b,b\}^{[d]}$. Let 
$D_\g$ and $Z_m^\g$ be as in 
\eqref{eq:D_ga}. 
Recall \eqref{eq:err-decomp} and let $\hat\De^\g:=\De(
L(Z_m^\g,U),D^\g)$,  
so that $\hat\De^\g\le b$. So, $\ii{\hat\De^\g>\vp}\ge\frac1{b-\vp}\,(\hat\De^\g-\vp)$, whence 
$\P(\hat\De^\g>\vp)\ge\frac1{b-\vp}\,(\E\hat\De^\g-\vp)$. 
Therefore, by 
\eqref{eq:R,U}, 
Theorem~\ref{th:average}, Proposition~\ref{lem:convex minor}, and Jensen's inequality, 
\begin{align*}
	\max_{\g\in\{-b,b\}^{[d]}}\P(\hat\De^\g>\vp)
	\ge 
	&\,\frac1{2^d}\sum_{\g\in\{-b,b\}^{[d]}}\P(\hat\De^\g>\vp) \\ 
	&\ge\frac1{b-\vp}\,\Big(\,
	\frac1{2^d}\sum_{\g\in\{-b,b\}^{[d]}}\E\hat\De^\g-\vp\Big) 
	\ge\frac{b\,\hatopt(\nu,b)-\vp}{b-\vp}. 
\end{align*}
Take now any $\nu_*\in[3,\infty)$ and any real $\nu\ge\nu_*$. 
Then, by Lemma~\ref{lem:hatopt} and \eqref{eq:C_i<}, 
$\hatopt(\nu,b)\ge\tfrac12\,\big(1-\big(\frac{i_{\nu_*} + 1}{i_{\nu_*}}\big)^{1/8}\psi_b(\nu)\big)$. 
Further, take any $z\in(0,\sqrt{\nu_*}]$ and $w\in(0,z)$, and then take $b=z/\sqrt\nu$ and $\vp=w/\sqrt\nu$, so that the conditions $b\in(0,1]$ and $\vp\in(0,b)$ assumed in the beginning of this proof do hold. 
Then 
\begin{align*}
	\max_{\g\in\{-1,1\}^{[d]}}\P\Big(\hat\De^\g>\frac w{\sqrt\nu}\Big)
	&\ge\frac1{z - w}\, \Big(\frac z2 \Big[1 - \Big(\frac{i_{\nu_*} + 1}{i_{\nu_*}}\Big)^{1/8} 
	\frac{\erf(z/\sqrt2)}{
	\exp\{- z^2/(6\nu_*)\}}\Big] - w\Big) \\ 
	&=:\plo(w,\nu_*,z).  
\end{align*}
It remains to note that $\plo(\frac1{\sqrt{320}},\frac{128}{10},\frac{331}{1000})>0.238$, $\plo(\frac1{\sqrt{320}},3,\frac{320}{1000})>0.227$, $\plo(\frac1{\sqrt{413/10}},\frac{128}{10},\frac{681}{1000})>0.01563>\frac1{64}$, and $\plo(\frac1{\sqrt{496/10}},3,\frac{601}{1000})>0.0159>\frac1{64}$.  
\end{proof}

\hide{
\begin{theorem}\label{th:asymp} 
One has 
\begin{equation}\label{eq:IS>hatopt}
	\inf_{L} \sup_{D}\E\Delta(\hat h_{m,L}^D,D)\sim\frac{c_\infty}{\sqrt\nu}    
\end{equation}
as $m$ and $d$ vary in any way such that $d\ge1$ and $\nu=m/d\to\infty$. 
\end{theorem}
 }

\section*{Acknowledgments} We are pleased to thank
Peter Gr\"unwald for bringing the result of \citet{audibert2009} to our attention, and
the referees for carefully reading the paper and useful suggestions on the presentation. 

\appendix

\section{Identities and inequalities for binomial distributions: details concerning the function \texorpdfstring{$\opt$}{
}}\label{app:bayes}

Recall the definition of $\opt$ in \eqref{eq:opt:=}. 
Take any $b\in[0,1]$ and $k\in\intr1\infty$. 
By \eqref{eq:s_k:=}, 
\begin{equation}\label{eq:s_k=sum}
	s_k(b)=\frac1{2^k}\sum_{i=0}^j\binom ki\big[(1+b)^{k - i}(1-b)^i - (1-b)^{k - i}(1+b)^i\big], 
\end{equation} 
\begin{equation}\label{eq:j}
j:=j_k:=\lfloor k/2\rfloor. 	
\end{equation}
Using identities $(k-i)\binom ki=k\binom{k-1}i$ and $i\binom ki=k\binom{k-1}{i-1}$, we have 
\begin{align*}
	\frac{2^k}k\,s'_k(b):=&\sum_{i=0}^j\binom{k-1}i\big[(1+b)^{k - i-1}(1-b)^i + (1-b)^{k - i-1}(1+b)^i\big]	\\ 	
	-&\sum_{i=1}^j\binom{k-1}{i-1}[(1+b)^{k - i}(1-b)^{i-1} + (1-b)^{k - i}(1+b)^{i-1}\big]. 
\end{align*}
Making in the second sum the substitution $i=r+1$, then replacing there $r$ back by $i$, and introducing 
\begin{equation}\label{eq:q}
	q:=q_k:=k - j-1, 
\end{equation}
we have 
\begin{equation}\label{eq:s'(y)}
	\frac{2^k}k\,s'_k(b)\Big/\binom{k-1}j=(1+b)^{q}(1-b)^j + (1-b)^{q}(1+b)^j
	=2(1-b^2)^{q},  
\end{equation}
which is non-increasing in $b\in[0,1]$. 
So, the function $s_k$ is concave. 

Moreover, it follows that 
\begin{equation}\label{eq:s_k=}
	s_k(b)=s'_k(0)S_q(b),\quad\text{where}\quad S_q(b):=\int_0^b(1-u^2)^{q}\,du. 
\end{equation}
For all $i\in\intr0\infty$, in view of \eqref{eq:j}, \eqref{eq:q}, and \eqref{eq:s'(y)}, 
\begin{equation}\label{eq:s'odd-even}
q_{2i+2}=q_{2i+1}=i	\quad\text{and}\quad s'_{2i+2}(0)=s'_{2i+1}(0)
\end{equation}
and hence, by \eqref{eq:s_k=}, one has the curious, and useful, identity 
\begin{equation}\label{eq:odd-even}
	s_{2i+2}(b)=s_{2i+1}(b). 
\end{equation}

We also have 

\begin{lemma}\label{lem:concave}
Take any $b\in(0,1)$. Then the function 
\begin{equation}\label{eq:bayes,0,1,3}
	\{0,1,3,5,\dots\}\ni k\mapsto\opt(k,b)
\end{equation}
is strictly convex. 
\end{lemma}

\begin{proof}
In view of \eqref{
eq:opt:=}, it is enough to show that the function 
$
\{0,1,3,5,\dots\}\ni k\mapsto s_k(b)
$
is strictly concave. 
By \eqref{eq:s_k=sum} and \eqref{eq:j}, 
\begin{equation}\label{eq:s_0,s_1,s_3}
	s_0(b)=0,\quad s_1(b)=b,\quad s_3(b)=\tfrac12\,(3b-b^3)<3s_1(b). 
\end{equation}
So, the restriction of the function in \eqref{eq:bayes,0,1,3} to the set $\{0,1,3\}$ is strictly concave. 

It remains to show that the restriction of this function to the set $\{1,3,5,\dots\}$ is strictly concave. 
Take any $i\in\intr0\infty$. 
We have to show that 
\begin{equation*}
	g(b):=
	\ts_i(b)+\ts_{i+2}(b)-2\ts_{i+1}(b)<0,   
\end{equation*}  
where 
\begin{equation}\label{eq:ts}
	\ts_i:=s_{2i+1}. 
\end{equation}
By 
\eqref{eq:s_k=}, 
$\ts'_\al(b)=\ts'_\al(0)(1-b^2)^{\al}>0$; here and in the rest of this proof, $\al$ stands for an arbitrary nonnegative integer. 
So, $g'(b)$ has the same sign as     
\begin{equation}
	\frac{g'(b)}{\ts'_{i+1}(b)}\,2(1 - b^2) (i+2) (2 i +3)
	=
	-1 - 2 (3 + 2 i) w + (15 + 16 i + 4 i^2) w^2=:g_1(w) 
\end{equation}
where $w:=b^2$. Since the function $g_1$ is convex, with $g_1(0)=-1<0$ and $g_1(1)=4 (2 + 3 i + i^2)>0$, it follows that $g_1(w)$ switches exactly once in sign,
from $-$ to $+$, as $w$ increases from $0$ to $1$. 
That is, $g(b)$ switches exactly once, from
decreasing to increasing,
as $b$ increases from $0$ to $1$. 
Also, by \eqref{eq:s_k=sum} and \eqref{eq:ts}, 
$\ts_\al(0)=0$ and $\ts_\al(1)=1$, whence $g(0)=0=g(1)$. Thus, indeed $g(b)<0$ for all $b\in(0,1)$. 
\end{proof}

\begin{lemma}\label{lem:C_i}
For $i\in\intr0\infty$, let 
\begin{equation}\label{eq:C_i}
	C_i:=\sqrt{\frac\pi2}\,\frac{s'_{2i+1}(0)}{\sqrt{2i+1}}
	=\frac{\sqrt{\pi(i+1/2)}}{2^{2i}}\,\binom{2i}i,  
\end{equation}
the latter equality following by \eqref{eq:s'(y)} and \eqref{eq:j}.  

Then $C_i$ decreases from $\sqrt{\frac\pi2}$ to $1$ as $i$ increases from $0$ to $\infty$, and for all $i\in\intr1\infty$  
\begin{equation}\label{eq:C_i<}
	C_i<\Big(\frac{i + 1}i\Big)^{1/8}. 
\end{equation}
\end{lemma}

\begin{proof}
In this proof, it is assumed that $i\in\intr0\infty$. Let 
\begin{equation*}
	r_i:=\frac{C_i}{C_{i+1}}=\frac{2i+2}{\sqrt{(2i+2)^2-1}}>1. 
\end{equation*}
So, $C_i$ indeed decreases in $i$. It is easy to check that $C_0=\sqrt{\frac\pi2}$ and $C_i\to C_\infty:=1$ as $i\to\infty$. 

It remains to verify inequality \eqref{eq:C_i<}. Accordingly, assume through the end of this proof that $i\in\intr1\infty$. Then 
\begin{equation*}
	-1+r_i^{-8}\Big/\Big(1 - \frac1{(i + 1)^2}\Big)
	=\frac{96 i^4+384 i^3+560 i^2+352 i+81}{256 i (i+1)^6 (i+2)}>0, 
\end{equation*}
which shows that 
\begin{equation*}
	r_i<\Big(1 - \frac1{(i + 1)^2}\Big)^{-1/8},
\end{equation*}
whence 
\begin{multline*}
	C_i=C_\infty\prod_{\al=i}^\infty r_\al=\prod_{\al=i}^\infty r_\al
	<\prod_{\al=i}^\infty \Big(1 - \frac1{(\al + 1)^2}\Big)^{-1/8} \\ 
	=\prod_{\al=i}^\infty \Big(\frac{\al}{\al+1}\Big/\frac{\al+1}{\al+2}\Big)^{-1/8}
	=\Big(\frac{i + 1}i\Big)^{1/8},  
\end{multline*} 
which completes the proof of Lemma~\ref{lem:C_i}. 
\end{proof}

\begin{lemma}\label{lem:hatopt}
Take any real $\ka\ge1$ and any $b\in[0,1]$. 
Then 
\begin{equation*}
\hatopt(\ka,b)\ge\tfrac12\,\big(1-C_{i_\ka}\psi_b(\ka)\big),  
\end{equation*}
where $i_\ka:=\lfloor\frac{\ka-1}2\rfloor$, $C_i$ as in \eqref{eq:C_i}, and 
\begin{equation*}
	\psi_b(\ka):=
	e^{b^2/6}\erf(b\sqrt{\ka/2}). 
\end{equation*}
\end{lemma}

\begin{proof}
For brevity, let $i:=i_\ka=\lfloor\frac{\ka-1}2\rfloor$ and $k:=2i+1$. Then $i\in\intr0\infty$, $k=2i+1\le\ka<2i+3=k+2$, and so, by \eqref{eq:hatopt:=}, 
\begin{equation}\label{eq:hatopt(ka)}
	\hatopt(\ka,b)=\tfrac{k+2-\ka}2\,\opt(k,b)+\tfrac{\ka-k}2\,\opt(k+2,b). 
\end{equation}
In view of \eqref{eq:j} and \eqref{eq:q}, $j_k=i=q_k=q$.
\label{log-conv}
It is well known and easily proved (
cf.\ e.g.\ \cite{kingman-matr}) that all mixtures of log-convex functions are log-convex. Therefore and because the function $q\mapsto e^{-qu^2}$ is log-convex, $\tS_q(b):=\int_0^b e^{-qu^2}\,du$ is log-convex in real $q$; that is, $f(q):=\ln\tS_q(b)$ is convex in $q$. 
\big[For readers' convenience, a quick and direct way to verify the convexity of $f$ is to see that for any real $q_1,q_2$ and any $t\in(0,1)$ the inequality 
$f((1-t)q_1+tq_2)\le(1-t)f(q_1)+tf(q_2)$ can be rewritten as the instance of H\"older's inequality 
$\int_0^b F_1(u)F_2(u)\,du\le(\int_0^b F_1(u)^p\,du)^{1/p}\,(\int_0^b F_2(u)^{\tilde p}\,du)^{1/\tilde p}$
with 
$F_1(u):=e^{-(1-t)q_1u^2}$, $F_2(u):=e^{-tq_2u^2}$, $p:=\frac1{1-t}$, and $\tilde p:=\frac1t$.\big]
So, $\ln\dfrac{\tS_{q+1/2}(b)\tS_0(b)}{\tS_{1/2}(b)\tS_q(b)}
=f(q+1/2)-f(q)-f(1/2)+f(0) 
=\int_0^{1/2}[f'(q+u)-f'(u)]\,du\ge0$ for $q\ge0$, since $f'(u)$ is increasing in $u$ 
\big(instead of the last equality here, one can also use inequality \cite[(3.17.5)]{HLP} with $\phi=f$, $h=x=q/2+1/4$ and $h'=q/2-1/4$ there\big). 
Recalling now \eqref{eq:s_k=} and using the elementary inequality $1-t\le e^{-t}$ for real $t$,  
we have 
\begin{equation}\label{eq:S_k<}
\begin{aligned}
	S_q(b)\le\tS_q(b)
	&\le\tS_{q+1/2}(b)\frac{\tS_0(b)}{\tS_{1/2}(b)} \\ 
	&=\frac{\sqrt\pi}2\,\frac{\erf(b\sqrt{q+1/2})}{\sqrt{q+1/2}}\,\frac{b}{\int_0^b e^{-u^2/2}\,du} \\ 
	&\le\frac{\sqrt\pi}2\,\frac{\erf(b\sqrt{q+1/2})}{\sqrt{q+1/2}}\,
	e^{b^2/6} \\ 
	&=\sqrt{\frac\pi2}\,\frac1{\sqrt k}\,
	e^{b^2/6}\,\erf(b\sqrt{k/2})
	=\sqrt{\frac\pi2}\,\frac1{\sqrt k}\,\psi_b(k);   
\end{aligned}	
\end{equation}
here we used the inequality $g(b):=\int_0^b e^{-u^2/2}\,du-be^{-b^2/6}>0$ for $b>0$, which follows because $g(0)=0$ and $g'(b)=e^{-b^2/6}(e^{-b^2/3}-(1-b^2/3))
>0$ for $b>0$. 
So, in view of \eqref{
eq:opt:=}, \eqref{eq:s_k=}, and \eqref{eq:C_i}, 
$\opt(k,b)\ge\tfrac12\,\big(1-C_i\psi_b(k)\big)$. 
Replacing here $k$ by $k+2$, one has 
\begin{equation*}
\opt(k+2,b)\ge\tfrac12\,\big(1-C_{i+1}\psi_b(k+2)\big)\ge\tfrac12\,\big(1-C_i\psi_b(k+2)\big); 	
\end{equation*}
the last inequality here follows because, by Lemma~\ref{lem:C_i}, $C_i$ decreases in $i$. 
Now \eqref{eq:hatopt(ka)} yields 
\begin{equation*}
	\hatopt(\ka,b)\ge\tfrac12-\tfrac12\,C_i
	\big[\tfrac{k+2-\ka}2\,\psi_b(k)+\tfrac{\ka-k}2\,\psi_b(k+2)\big]
	\ge \tfrac12-\tfrac12\,C_i\psi_b(\ka),  
\end{equation*}
the latter inequality following by the concavity of $\psi_b(u)$ in $u\ge0$. 
This completes the proof of Lemma~\ref{lem:hatopt}. 
\end{proof}


\section{Simplified form of \texorpdfstring{$c_{m,d}^\LB$}{}}
\label{c^LB reduction}
Take any finite set $\X$. Take then any 
hypothesis class $\H$ with $\vc\H=d$. 
Let now $\tX$ be any subset of $\X$ of cardinality $d$ such that $\tX$ is shattered by $\H$. 
Clearly, for any learning algorithm 
$L\in\L_{\X;m,\H}$ we have 
\begin{equation}\label{eq:reduct}
	\sup_{D\in\D_{\X}}\RR_m(L,D)\ge\sup_{D\in\D_{\X;\tX}}\RR_m(L,D), 
\end{equation}
where  
$\D_{\X;\tX}$ is the set of
all distributions $D\in\D_{\X}$ with support contained in the set $\tX\times\Y$ (and, recall, $\D_\X$ is the set of all distributions on $\X\times\Y$). 
Next, any distribution $D\in\D_{\X;\tX}$
may be identified with the corresponding distribution, say $\tilde D$, in $\D_{\tX}$. 
Also, for any $D\in\D_{\X;\tX}$, the value of $\RR_m(L,D)$ depends on $L$ only through its restriction, say $\tilde L$, to $(\tX\times\Y)^k$ --- so that $\RR_m(L,D)=\RR_m(\tilde L,\tilde D)$. 
Moreover, we may identify the finite set $\X$ with the set $[N]=\{1,\dots,N\}$ for some natural $N\ge d$, and then without loss of generality $\tX=[d]=\{1,\dots,d\}$.  
Thus, 
for any finite set $\X$ and any hypothesis class $\H$ with $\vc\H=d$, it follows from \eqref{eq:reduct} that 
\begin{equation}\label{eq:relabel}
\inf_{L\in\L_{\X;m,\H}}\,\sup_{D\in\D_{\X
}}\,\RR_m(L,D)
\ge\inf_{\tilde L\in\L_{[d];m,\Y^{[d]}}}\,\sup_{\tD\in\D_{[d]
}}\RR_m(\tilde L,\tD). 
\end{equation}
 
The right-hand side (RHS) of \eqref{eq:relabel} does not depend on $\X$ or $\H
$, as long as $\vc\H=d$. 
So, by \eqref{eq:ub,lb}, $c_{m,d}^\LB
/\sqrt{m/d}\ge
\text{r.h.s.\ of \eqref{eq:relabel}}$. The reverse of the latter inequality follows trivially from \eqref{eq:ub,lb}, since the set $[d]$ is finite and 
$\vc{\Y^{[d]}}=d$. 
We conclude that $c_{m,d}^\LB
/\sqrt{m/d}=
\text{r.h.s.\ of \eqref{eq:relabel}}$. 
Thus, \eqref{eq:clb} follows.

%
	
\bibliographystyle{imsart-number}

\bibliography{refs1}

\end{document}